\documentclass[12pt]{article}
\usepackage[top=25truemm,bottom=25truemm,left=25truemm,right=25truemm]{geometry}
\usepackage{amsmath,amssymb,amsthm,bm,graphicx,algorithm,algorithmic}
\allowdisplaybreaks

\title{$\ell_1$-Regularized ICA: A Novel Method for Analysis of Task-related fMRI Data}

\author{Yusuke Endo\thanks{Department of Mechanical Systems Engineering, Graduate School of Science and Engineering, Ibaraki University} 
\and Koujin Takeda\thanks{Department of Mechanical Systems Engineering, Graduate School of Science and Engineering, Ibaraki University, e-mail: koujin.takeda.kt@vc.ibaraki.ac.jp}}

\newtheorem{thm}{Theorem}
\newtheorem*{thm*}{Theorem}

\newcommand{\argmin}{\mathop{\rm arg~min}\limits}

\date{17th October, 2024}

\begin{document}
\maketitle

\begin{abstract}
We propose a new method of independent component analysis (ICA) in order to extract appropriate features from high-dimensional data. In general, matrix factorization methods including ICA have a problem regarding the interpretability of extracted features. For the improvement of interpretability, it is considered that sparse constraint on a factorized matrix is helpful. With this background, we construct a new ICA method with sparsity. 
In our method, the $\ell_1$-regularization term is added to the cost function of ICA, and minimization of the cost function is performed by difference of convex functions algorithm. For the validity of our proposed method, we apply it to synthetic data and real functional magnetic resonance imaging data.
\end{abstract}

Keywords: independent component analysis, sparse matrix factorization, sparse coding, difference of convex functions algorithm, functional magnetic resonance imaging

\section{Introduction}
In machine learning, matrix factorization (MF) is known as significant unsupervised learning method for feature extraction from data, and various methods of MF have been proposed. In MF, two factorized matrices $\bm{A} \in \mathbb{R}^{p \times K}$ and $\bm{S} \in \mathbb{R}^{K \times N}$ are computed from an observed data matrix $\bm{X} \in \mathbb{R}^{p \times N}$, which are related by
$\bm{X}\approx \bm{AS}$. 
As the methods of MF, various independent component analysis (ICA) \cite{ICA}, principal component analysis (PCA), and non-negative MF \cite{NMF} are widely known and frequently used. However, these MF methods have a common problem in the application to high-dimensional data: it is generally difficult to interpret what extracted features mean. To resolve this problem, sparse MF methods \cite{MOD}, \cite{KSVD} are considered to be useful, where a sparse constraint is imposed on either factorized matrix $\bm{A}$ or $\bm{S}$. By introducing sparse constraint into PCA \cite{SparsePCA1}, \cite{SparsePCA2} or non-negative MF \cite{SparseNMF}, it is found that the extracted feature can be interpreted relatively easily due to sparsity. However, as far as the authors know, the widely-accepted sparse method for ICA has not been established yet.

Due to its significance, there are many applications of MF to real-world data. Analysis of biological data such as neuronal activity or electrocardiogram is one of such applications. Among various MF methods and applications, we especially focus on the application of ICA to functional magnetic resonance imaging (fMRI) data. Nowadays, fMRI is a significant experimental method for understanding the function of the whole human brain. In fMRI, time series data of cerebral blood flow in test subject is measured by 
magnetic resonance imaging. There are two types of fMRI: one is resting-state fMRI, which is measured under no human's task or stimulus, and the other is task-related fMRI, which is measured under some tasks or stimuli. In the application of MF to fMRI \cite{MF2fMRIexA}, \cite{MF2fMRIexB}, \cite{MF2fMRIexC}, \cite{MF2fMRIexD}, the dimensions $p$ and $N$ are the numbers of voxels and scanning time steps, respectively. In this setting, $K$ vectors in the spatial feature factorized matrix $\bm{A}$ and temporal feature matrix $\bm{S}$ are interpreted as $K$ features in the human brain activity. 

For feature extraction from task-related fMRI by MF, 
it is considered that statistical independence among extracted feature signals is generally significant, in particular for temporal feature matrix $\bm S$. However, it will be better to impose an additional constraint such as sparsity on spatial feature matrix $\bm A$ for localization of the elements in $\bm A$. This is because it is widely believed that information representation for external stimuli in the human brain is sparse. Therefore, it is not enough for feature extraction
only to assume statistical independence for temporal feature $\bm S$.
Practically, in the application of ICA to fMRI data, statistical independence is often imposed on $\bm A$ instead of $\bm S$ for localization of elements in $\bm A$. Such ICA is specifically called spatial ICA \cite{ICA2rsfMRI}. In contrast, the method with the assumption of statistical independence on temporal feature $\bm S$ is called temporal ICA. In this connection, it is shown recently that sparse MF or spatial ICA can extract more appropriate features from fMRI data than other MF methods without localization of the elements 
in $\bm A$ \cite{ET2024}. From these facts, we expect that temporal ICA can be improved by adding a sparse constraint on the matrix $\bm A$. 

With this background, we propose a novel ICA method, which may extract features appropriately from task-related fMRI.  In our method, we minimize the cost function of ICA with an additional regularization term. Namely, we impose the constraints of statistical independence on $\bm{S}$ and sparsity on $\bm{A}$. For minimizing the proposed cost function, we apply difference of convex functions (DC) algorithm \cite{TAO1986249,tao1988duality,DC1}, which is a powerful tool in nonconvex optimization problem. 

The organization of this article is as follows. In section \ref{sec:Model}, 
the framework of ICA is overviewed and several related works are presented.
In section \ref{sec:L1ICA},
we provide how to construct our proposed method and discuss the convergence of our algorithm theoretically. In section \ref{sec:Numerical Experiment}, we apply our method both to synthetic data and real-world data, then discuss the validity. Finally, section \ref{sec:Conclusion} is devoted to conclusions and future works.

\section{Overview of ICA}\label{sec:Model}
Throughout this article, matrices are denoted by bold uppercase letters such as $\bm{A}$ and vectors as bold lowercase letters such as $\bm{a}$. Scalars are denoted as small italic, and elements of vector/matrix are written as italic lowercase letter such as $[a_1, \ldots , a_n]$.

\subsection{Framework of ICA and FastICA}
In the introduction, we mainly focus on the application of ICA to fMRI. In practice, ICA is applied not only to biological data analysis but also to other various problems such as image processing or finance. Therefore, we formulate our proposed ICA in general setting for the application to various problems.

Originally, ICA is proposed as one of the blind source separation methods. The objective of all ICA methods is to estimate latent independent source $\bm{S}=[\bm{s}_1, \ldots, \bm{s}_K]^T$ and mixing matrix $\bm{A}$ simultaneously without any prior knowledge on them. In many algorithms of ICA, non-gaussianity is the quantity to be maximized.

Among various ICA methods, FastICA \cite{FastICA} is a widely used algorithm. In FastICA, factorized matrices $\bm{A}, \bm{S}$ are estimated as follows.

\begin{enumerate}
\item The observed matrix $\bm{X}$ is prewhitened by $\bm{Q} \in \mathbb{R}^{K \times p}$ as $\bm{Z} = \bm{QX}$, where $\bm{Q}$ denotes the prewhitening matrix and $\bm{Z}$ denotes the prewhitened observed matrix.
\item Next, the projection matrix $\bm{W}\in \mathbb{R}^{K \times K}$ is evaluated, which maximizes non-gaussianity among rows of the signal source matrix $\bm{S} = \bm{W}^T \bm{Z}$. The projection matrix $\bm{W}$ is frequently called unmixing matrix. The cost function for non-gaussianity is defined by
\begin{equation}\label{eq:negentoropy}
J_f(\bm{w}_i) = \frac{1}{N^2} \left\{ \sum_{j=1}^{N} G(\bm{w}_{i}^{T} \bm{z}_j) - \sum_{j=1}^{N} G(\nu_j) \right\}^2.
\end{equation}
In the above equation, $G(.)$ is a nonlinear function and defined by $G(y) = - e^{-y^2/2}$. The variable $\nu_j$ 
$(j \in \{ 1, \ldots, N \}$) is a normal random variable with zero mean and unit variance. In maximizing $J_f (\bm{w}_i)$, a fixed point algorithm in equation (\ref{eq:fixed point algorithm}) is used.
\begin{equation}\label{eq:fixed point algorithm}
\bm{w}_i \leftarrow \frac{1}{N} \sum_{j=1}^{N}\bm{z}_j G'(\bm{w}_{i}^{T} \bm{z}_j) - G''(\bm{w}_{i}^{T} \bm{z}_j)\bm{w}_i.
\end{equation}
In addition, the orthogonal constraint $\bm{W}^T\bm{W} = \bm{I}_K$ is imposed on the projection matrix $\bm{W}$, where $\bm{I}_K$ is $K$-dimensional identity matrix. In practice, the column vectors of $\bm{W}$ are mutually orthogonalized by Gram-Schmidt method, 
where the column vectors are updated as follows,
\begin{equation}\label{eq:Gram-Schmidt}
\bm{w}_i \leftarrow \bm{w}_i - \sum_{j=1}^{i-1} (\bm{w}_{i}^{T}\bm{w}_j)\bm{w}_j.
\end{equation}	
\item The mixing matrix $\bm{A}$ is calculated as $\bm{A} = \bm{Q}^{\dagger} \bm{W}$, where $\bm{Q}^{\dagger}$ is Moore-Penrose pseudo-inverse matrix of $\bm{Q}$. Accordingly, the observed matrix $\bm{X}$, the mixing matrix $\bm{A}$, and signal source matrix $\bm{S}$ are related by $\bm{X} \approx \bm{A}\bm{S}$, which is shown by the orthogonality of $\bm{W}$.
\end{enumerate}

\subsection{Related works on ICA with sparse constraint}
Before going into our proposed method, we should mention several related works regarding ICA with sparse constraint. 
First, there is a method called SparseICA \cite{zibulevsky2001blind1}, \cite{zibulevsky2001blind2},
which was introduced for blind source separation in image processing in real-world application \cite{bronstein2005sparse}.
In this method, independent source matrix $\bm{S}$ is computed as the product of the sparse coefficient matrix $\bm{C}$ and the 
dictionary matrix $\bm{\Phi}$ by $\bm{S} = \bm{C}\bm{\Phi}$. 
 
Next, ICA with regularization terms was proposed for the linear non-gaussian acyclic model \cite{harada2021sparse}. In this method, the regularization term for the separation matrix defined by $\bm{W}^T\bm{Q}$ 
is introduced. It should be noted that the construction of sparse matrix in ICA
is similar to our proposed method explained in the following section. However, the sparse constraint on the matrix is different. 
In this method, the sparse constraint is imposed on the separation matrix $\bm{W}^T\bm{Q}$, while the estimated mixture matrix $\bm{A}$ is sparse in our method. Similar studies with sparse separation matrix $\bm{W}^T\bm{Q}$ can also be found in other references \cite{chen2019sparse}, \cite{zhang2006ica}.

There is also other related work \cite{donnat2019constrained}, where the sparse constraint is imposed on the mixing matrix $\bm{A}$ like our method. Additionally, the objective of the work is to identify sparse neuronal networks like ours. 
However, there is a difference in the methodology: this method is based on Bayesian framework and the algorithm is
a stochastic one. 
More precisely, the mixing matrix $\bm{A}$ and parameters in the model are assumed to follow some probabilistic distributions
 and optimized by Markov chain Monte Carlo method. In contrast, our method is a deterministic one, where the cost function is minimized by DC algorithm.

\section{Our method}\label{sec:L1ICA}
In this section we elucidate the detail of our proposed method.

\subsection{Formulation}
In our method, the $i$-th vector in the projection matrix $\bm{w}_i$ minimizing cost function 
in equation (\ref{eq:proposed cost function}) must be evaluated.

\begin{equation}\label{eq:proposed cost function}
\begin{split}
J_o(\bm{w}_i) &= -J_f(\bm{w}_i) + \alpha \| \bm{Q}^{\#} \bm{w}_i \|_1 \\
&= \widetilde{J_f}(\bm{w}_i) + \alpha \|\bm{Q}^{\#} \bm{w}_i\|_1.
\end{split}
\end{equation}
Here $\| \cdot \|_p$ is $\ell_p$-norm, $\widetilde{J_f}(\bm{w}_i) = -{J_f}(\bm{w}_i)$, and $\bm{Q}^{\#} \in \mathbb{R}^{p \times K}$ is given by the minimizer in equation (\ref{eq:definition of Q sharp}).
\begin{equation}\label{eq:definition of Q sharp}
\bm{Q}^{\#} = \argmin_{\bm{Q}'}{\|\bm{Q}^{\dagger} - \bm{Q}'\|_{\rm F}} \quad \mathrm{s.t.} \quad \|\bm{Q}'\|_0 \leq k_0,
\end{equation}
with $\| \cdot \|_{\rm F}$ being Frobenuis norm and
$\| \cdot \|_0$ being the matrix $\ell_0$-norm or the number of non-zero elements in the matrix, respectively.
By definition, the minimizer matrix $\bm{Q}^{\#}$ is close to the original $\bm{Q}^{\dagger}$ and sparse with only $k_0$ nonzero elements. 
We need the sparse estimated mixture matrix $\bm{A}=\bm{Q}^{\#}\bm{W}$, whose sparsity depends on the sparsity of $\bm{Q}^{\#}$ \cite{FOUCART2023441}. Hence, approximate sparse inverse whitening matrix $\bm{Q}^{\#}$ is favorable rather than the original inverse whitening matrix $\bm{Q}^{\dagger}$ in our proposed method. In practical computation, the algorithm of orthogonal matching pursuit \cite{OMP1}, \cite{OMP2} is applied to obtain each row of $\bm{Q}^{\#}$.

We apply widely-used DC algorithm for minimization of the cost function. 
DC algorithm is applicable to the minimization problem, where the cost functions $J(\bm{w})$ is represented by the difference of two convex functions $g(\bm{w}), h(\bm{w})$ as
\begin{equation} \label{eq:difference of convex representation}
J(\bm{w}) = g(\bm{w}) - h(\bm{w}).
\end{equation}
Note that Hessians of the two convex functions must be bounded above.
In DC algorithm, the following two processes are repeated until convergence or maximum iteration number $M$ is reached.
\begin{enumerate}
\renewcommand{\labelenumi}{(\roman{enumi})}
 \item Update ${}_{\nabla} \bm{h}^{(t)}$ by choosing one element from the subgradient set of $h(\bm{w})$, namely 
 ${}_{\nabla} \bm{h}^{(t)}\in \partial h(\bm{w}^{(t)})$ with $\bm{w}^{(t)}$ being the $t$-th update of $\bm{w}$.
 \item Update $\bm{w}^{(t+1)}$ by $\bm{w}^{(t+1)} =  \argmin_{\bm{w}} \{g(\bm{w})-({}_{\nabla} \bm{h}^{(t)})^T \bm{w} \}$.
\end{enumerate}
It is considered that the convergent solution by DC algorithm is often a global optimal solution. Therefore, DC algorithm is one of the effective approaches for nonconvex optimization.

For application of DC algorithm to minimization in our method, 
the cost function $J_o(\bm{w}_i)$ must be expressed by the difference of two convex functions. First, Lipschitz constant $L$ of $\widetilde{J_f}(\bm{w}_i)$ is introduced, which can be defined because $\widetilde{J_f}(\bm{w}_i)$ is secondary differentiable. Then, the function 
$(L/2)\| \bm{w}_i \|_{2}^{2} - \widetilde{J_f}(\bm{w}_i)$ is easily proved to be convex by the definition of Lipschitz constant $L$. Moreover, regularization term $\|\bm{Q}^{\#} \bm{w}_i\|_1$ is clearly convex, and the sum of two convex functions $(L/2)\|\bm{w}_i\|_{2}^{2} + \alpha \|\bm{Q}^{\#} \bm{w}_i\|_1$ is also convex. Therefore, $J_o(\bm{w}_i)$ can be represented by the difference of two convex functions $g(\bm{w}_i), h(\bm{w}_i)$
like $J_o(\bm{w}_i) = g(\bm{w}_i) - h(\bm{w}_i)$, where the two convex functions are defined in equations 
(\ref{eq:convex function 1}) and (\ref{eq:convex function 2}).
\begin{eqnarray} 
g(\bm{w}_i) &=& \frac{L}{2} \|\bm{w}_i\|_{2}^{2} + \alpha \|\bm{Q}^{\#} \bm{w}_i\|_1, \label{eq:convex function 1}\\
h(\bm{w}_i) &=& \frac{L}{2} \|\bm{w}_i\|_{2}^{2} - \widetilde{J_f}(\bm{w}_i).  \label{eq:convex function 2}
\end{eqnarray}
Note that $g(\bm{w}_i)$ and $h(\bm{w}_i)$ are continuous functions in our $\ell_1$-regularized ICA.

In DC algorithm, computation of Lipschitz constant $L$ needs maximum eigenvalue of Hessian $\nabla^{2} \widetilde{J_f}(\bm{w}_i)$, whose computational complexity is $O(K^3)$. In practice, $L$ is evaluated by backtracking line search algorithm \cite{DC2}. In this algorithm, 
we prepare the tentative Lipschitz constant $l^{(t)}$ in the $t$-th update, then
$l^{(t)}$ is accepted as the final Lipschitz constant 
if the criterion of inequality (\ref{ineq:backtracking line search}) 
is satisfied for $\zeta \in (0, 1)$. Otherwise, 
$l^{(t)}$ is rescaled by $l^{(t)} \leftarrow \eta l^{(t)}$ with the constant $\eta>1$, then inequality (\ref{ineq:backtracking line search}) is checked again.
\begin{equation} \label{ineq:backtracking line search}
J_o(\bm{w}_{i}^{(t+1)}) \leq J_o(\bm{w}_{i}^{(t)}) - \frac{\zeta}{2} \|\bm{w}_{i}^{(t+1)} - \bm{w}_{i}^{(t)}\|_{2}^{2}.
\end{equation}   

After evaluation of Lipschitz constant, we move on to the formulation of DC algorithm in our proposed method.
In DC algorithm, subderivative $\partial h(\bm{w}_i)$ is necessary for the calculation of ${}_{\nabla} \bm{h}^{(t)}$
in the process (i), which is given by equation (\ref{eq:gradient of r}),
\begin{equation}\label{eq:gradient of r}
{}_{\nabla} \bm{h}^{(t)} = L \bm{w}_{i}^{(t)} - \nabla \widetilde{J_f}(\bm{w}_{i}^{(t)}).
\end{equation}
Then, the process (ii) for the update of $\bm{w}_{i}^{(t)}$ is represented as follows.
\begin{equation} \label{eq:update next point}
\begin{split}
\bm{w}_{i}^{(t+1)} &= \argmin_{\bm{w}} \left \{g(\bm{w}) - ({}_{\nabla} \bm{h}^{(t)})^T \bm{w} \right \} \\
&= \argmin_{\bm{w}} \left \{ \frac{L}{2} \|\bm{w}\|_{2}^{2} + \alpha \|\bm{Q}^{\#} \bm{w}\|_1 - ({}_{\nabla} \bm{h}^{(t)})^T \bm{w} \right \}.
\end{split}
\end{equation}
The minimization problem in equation (\ref{eq:update next point}) can be regarded as a generalized lasso problem \cite{gLasso}, 
whose solution cannot be obtained analytically. 
Hence, we attempt to find the solution numerically by the method of alternating directions method of multipliers (ADMM) \cite{ADMM}. 
In the application of ADMM, $\bm{w}_{i}^{(t+1)}$ is computed by minimizing augmented Lagrangian 
$\mathcal{L}_{\rho}(\bm{w}, \bm{\gamma} , \bm{\tau})$,
\begin{equation} \label{eq:augmented Lagrangian}
\mathcal{L}_{\rho}(\bm{w}, \bm{\gamma}, \bm{\tau}) = \frac{L}{2} \|\bm{w}\|_{2}^{2} - ({}_{\nabla} \bm{h}^{(t)})^T \bm{w} 
+ \alpha \| \bm{\gamma} \|_1 + \bm{\tau}^T (\bm{Q}^{\#} \bm{w} - \bm{\gamma}) + \frac{\rho}{2} \|\bm{Q}^{\#} \bm{w} - \bm{\gamma} \|_{2}^{2}.
\end{equation}
The variables $\bm{w}, \bm{\gamma}, \bm{\tau}$ minimizing $\mathcal{L}_{\rho}(\bm{w}, \bm{\gamma} , \bm{\tau})$ are alternately updated by equations (\ref{eq:update w in ADMM}), (\ref{eq:update gamma in ADMM}) and (\ref{eq:update tau in ADMM}) until convergence.
\begin{eqnarray} 
\bm{w}^{[\theta + 1]} &=& \argmin_{\bm{w}}{\mathcal{L}_{\rho}(\bm{w}, \bm{\gamma}^{[\theta]}, \bm{\tau}^{[\theta]})}, \label{eq:update w in ADMM} \\
\bm{\gamma}^{[\theta + 1]} &=& \argmin_{\bm{\gamma}}{\mathcal{L}_{\rho}(\bm{w}^{[\theta + 1]}, \bm{\gamma}, \bm{\tau}^{[\theta]})}, \label{eq:update gamma in ADMM} \\
\bm{\tau}^{[\theta + 1]} &=& \bm{\tau}^{[\theta]} + \rho(\bm{Q}^{\#} \bm{w}^{[\theta+1]} - \bm{\gamma}^{[\theta+1]}),  \label{eq:update tau in ADMM}
\end{eqnarray}
where superscript $[ \theta ]$ means the variable of the $\theta$-th update in ADMM iteration.
Update of $\bm{w}^{[\theta + 1]}$ can be rewritten by putting the solution of $\partial \mathcal{L}_{\rho}(\bm{w}, \bm{\gamma}^{[\theta]}, \bm{\tau}^{[\theta]}) / \partial \bm{w} = \bm{0}$,
which leads to 
\begin{equation} \label{eq:update solution w in ADMM}
\bm{w}^{[\theta + 1]} = (L \bm{I}_K + \rho {\bm{Q}^{\#}}^T \bm{Q}^{\#} )^{-1} \left \{ {}_{\nabla} \bm{h} + {\bm{Q}^{\#}}^{T} (\rho \bm{\gamma}^{[\theta]} - \bm{\tau}^{[\theta]}) \right \}.
\end{equation}
Similarly, by solving equation $\partial \mathcal{L}_{\rho}(\bm{w}^{[\theta + 1]}, \bm{\gamma}, \bm{\tau}^{[\theta]}) / \partial \bm{\gamma} = \bm{0}$ analytically, update of the $j$-th element in $\bm{\gamma}^{[\theta + 1]}$ is represented as follows, 
\begin{equation} \label{eq:update solution gamma in ADMM}
\gamma_{j}^{[\theta+1]} = \mathrm{S}_{\frac{\alpha}{\rho}} \left \{ (\bm{Q}^{\#})_{j,} \bm{w}^{[\theta + 1]} + \frac{\tau_j^{[\theta]}}{\rho} \right \},
\end{equation}
where $(\bm{Q}^{\#})_{j,}$ means the $j$-th row vector of $\bm{Q}^{\#}$ and $\mathrm{S}_{a}(x)$ is soft thresholding function defined in equation (\ref{eq:soft thresholding function}),
\begin{equation} \label{eq:soft thresholding function}
\mathrm{S}_{a}(x) = \begin{cases}
x - a & $\rm for$ \ \ x > a,\\
0 & $\rm for$ \ \ -a \leq x \leq a,\\
x + a & $\rm for$ \ \ x < -a.
\end{cases}
\end{equation}

Finally, by combining these results, our proposed algorithm is obtained, which is summarized in Algorithm \ref{al:proposed_method}.

\begin{algorithm}[t]
\caption{The algorithm of $\ell_1$-regularized ICA}
\label{al:proposed_method}
\begin{algorithmic}
\renewcommand{\algorithmicrequire}{\textbf{Input:}}
\renewcommand{\algorithmicensure}{\textbf{Output:}}
\ENSURE factorized matrix: $\bm{A} \in \mathbb{R}^{p \times K}$, $\bm{S} \in \mathbb{R}^{K \times N}$
\\
\STATE compute $\bm{Q}, \bm{Q}^{\#}$ and prepare pre-whitened $\bm{X}$
\STATE prepare initial matrix $\bm{W}$ and normal random variables $\bm{\nu}$ 
\STATE normalize each column of $\bm{W}$
\\
\FOR {$i=1$ to $K$}
\WHILE {$t<M$}
\STATE update ${}_{\nabla} \bm{h}^{(t)}$ by equation (\ref{eq:gradient of r})
\WHILE {until convergence}
\STATE update $\bm{w}^{[\theta + 1]}$, $\bm{\gamma}^{[\theta + 1]}$, $\bm{\tau}^{[\theta + 1]}$ by equations (\ref{eq:update solution w in ADMM}) (\ref{eq:update solution gamma in ADMM}) (\ref{eq:update tau in ADMM})
\STATE $\theta \leftarrow \theta + 1$
\ENDWHILE 
\STATE $\bm{w}_{i}^{(t+1)} \leftarrow \bm{w}_{i}^{[\theta+1]}$
\STATE orthogonalize and normalize $\bm{w}_{i}^{(t+1)}$
\IF {$l^{(t)}$ does not satisfy criterion in inequality (\ref{ineq:backtracking line search}) }
\STATE $l^{(t)} \leftarrow \eta l^{(t)}$
\STATE update ${}_{\nabla} \bm{h}^{(t)}$ and $\bm{w}_{i}^{(t)}$ by equations (\ref{eq:gradient of r}) (\ref{eq:update next point})
\ENDIF
\ENDWHILE
\ENDFOR
\STATE compute $\bm{A} = \bm{Q}^{\#} \bm{W}$, $\bm{S} = \bm{W}^T \bm{Q}\bm{X}$
\STATE output $\bm{A}$ and $\bm{S}$
\end{algorithmic}
\end{algorithm}

\subsection{Convergence analysis}
We can derive a convergence condition for our proposed algorithm for $\ell_1$-regularized ICA. The theorem 
and the proof of the convergence condition are given in the following.
 
For general DC algorithm, the condition of non-increasing function $J_o(\bm{w}_{i}^{(t+1)}) \leq J_o(\bm{w}_{i}^{(t)})$ must be satisfied for the updated variable $\bm{w}_{i}^{(t+1)}$ \cite{tao1997convex}. When this condition is satisfied,
it is guaranteed that the sequence $\{ \bm{w}^{(t)} \}_{t \in \mathbb{Z}_{\ge 0}}$ converges to a point $\bm w^*$ satisfying
the condition for subgradient as 
\begin{equation} \label{eq:convpoint}
\varnothing \neq \partial g(\bm{w}^*) \cap \partial h(\bm{w}^*).
\end{equation}
The significant difference of DC algorithm in our $\ell_1$-regularized ICA is orthogonalization and normalization after variable update, which requires a slight change in the proof for convergence condition of DC algorithm.
For this purpose, we should distinguish the variable after orthogonalization and normalization, 
which is denoted by $\widehat{\bm{w}}_{i}^{(t+1)}$. With this variable,
we can show that our proposed algorithm converges to a point in equation (\ref{eq:convpoint}) 
if the inequality $J_o(\widehat{\bm{w}}_{i}^{(t+1)}) \leq J_o(\bm{w}_{i}^{(t)})$ is satisfied. 

The convergence theorem for our proposed algorithm is stated as follows.
It should be commented that the convergence condition in the theorem is a sufficient condition, which can be relaxed.

\begin{thm}\label{thm:convergence}
If the following condition is satisfied at each step of $(t)$, the sequence
$\{ \bm{w}^{(t)} \}_{t \in \mathbb{Z}_{\ge 0}}$ converges to a certain point by our proposed algorithm.
\begin{equation}\label{eq:conditionin1}
\left\{
\begin{split}
C \leq 1, &\quad {if}\ \min_{\bm w} \{g(\bm{w}) - ({}_{\nabla}\bm h^{(t)})^T \bm{w}\} > 0,\\
C \geq 1, &\quad {if}\ \min_{\bm w} \{g(\bm{w}) - ({}_{\nabla}\bm h^{(t)})^T \bm{w}\} < 0,
\end{split}
\right.
\end{equation}
where $C$ is a constant defined by
\begin{eqnarray}\label{eq:defC}
 C &=& \max \left\{ \sqrt{ \lambda_{\mathrm{max}} \left( (\bm{R}_i^{(t+1)})^T \bm{R}_i^{(t+1)} \right)}, \right.
 \nonumber \\
 && \!\!\!\!\!
 \left. \sqrt{ \frac{ N \lambda_{\mathrm{max}}\left({\bm{Q}^{\#}}^T \bm{Q}^{\#}\right)}{\lambda_{\rm min} \left({\bm{Q}^{\#}}^T \bm{Q}^{\#}\right)} } \frac{1}{\| \bm{w}_{i}^{(t+1)} \|_2},\
 \left( \frac{  ({}_{\nabla} \bm{h}^{(t)})^T \bm{w}_{i}^{(t+1)} }{\| {}_{\nabla} \bm{h}^{(t)} \|_2 \| \bm w_i^{(t+1)} \|_2} \right)^{-1} \!\!\!\!\!\!
 \frac{1}{\| \bm{w}_{i}^{(t+1)} \|_2} \right\}\!\! .
\end{eqnarray}
  In this definition,
   the operation of orthogonalization is expressed by the multiplication of a square matrix $\bm{R}_i^{(t+1)} \in \mathbb{R}^{K \times K}$, namely $\bm{R}_i^{(t+1)} \bm{w}_{i}^{(t+1)}$, and
  $\lambda_{\mathrm{max}}(\cdot), \lambda_{\mathrm{min}}(\cdot)$ are maximum and minimum eigenvalues 
  of square matrix in the parenthesis, respectively.
\end{thm}

\begin{proof}
From the convexity of the function $h$,
\begin{equation}\label{eq:convexity}
h(\widehat{\bm{w}}_{i}^{(t+1)}) \geq h(\bm{w}_{i}^{(t)}) + ({}_{\nabla} \bm{h}^{(t)})^T (\widehat{\bm{w}}_{i}^{(t+1)} - \bm{w}_{i}^{(t)}).
\end{equation}

The $i$-th vector
$\bm{w}_{i}^{(t+1)}$ is orthogonalized 
to the set of vectors $\{ \bm{w}_{j}^{(t+1)} \}_{j \in \{ 1,\ldots,i-1 \} }$
first, then normalized. After orthogonalization, the normalized vector $\widehat{\bm{w}}_{i}^{(t+1)}$ is calculated as $\widehat{\bm{w}}_{i}^{(t+1)} = \frac{\bm{R}_i^{(t+1)} \bm{w}_{i}^{(t+1)}}{\| \bm{R}_i^{(t+1)}  \bm{w}_{i}^{(t+1)} \|_2}$. 
For this vector $\widehat{\bm{w}}_{i}^{(t+1)}$ , the upper bound of $J_o(\widehat{\bm{w}}_{i}^{(t+1)})$ is estimated as 
\begin{eqnarray}\label{eq:convergence-orthogonalize}
J_o(\widehat{\bm{w}}_{i}^{(t+1)}) &=& g(\widehat{\bm{w}}_{i}^{(t+1)}) - h(\widehat{\bm{w}}_{i}^{(t+1)}) \nonumber \\
&\overset{(\ref{eq:convexity})}{\leq}& g(\widehat{\bm{w}}_{i}^{(t+1)}) - \left\{ h(\bm{w}_{i}^{(t)}) + ({}_{\nabla} \bm{h}^{(t)})^T (\widehat{\bm{w}}_{i}^{(t+1)} - \bm{w}_{i}^{(t)}) \right\}
\nonumber \\
&=& g(\widehat{\bm{w}}_{i}^{(t+1)}) - ({}_{\nabla} \bm{h}^{(t)})^T \widehat{\bm{w}}_{i}^{(t+1)} - h(\bm{w}_{i}^{(t)}) + ({}_{\nabla} \bm{h}^{(t)})^T \bm{w}_{i}^{(t)}
\nonumber \\
&\overset{(\ref{eq:convex function 1})}{=}& 
\frac{L}{2} \| \bm{R}_i^{(t+1)} \bm{w}_{i}^{(t+1)} \|_2
+ \alpha \frac{\| \bm{Q}^{\#} \bm{R}_i^{(t+1)}  \bm{w}_{i}^{(t+1)}\|_1}
{\| \bm{R}_i^{(t+1)}  \bm{w}_{i}^{(t+1)} \|_2} - \frac{({}_{\nabla} \bm{h}^{(t)})^T \bm{R}_i^{(t+1)} \bm{w}_{i}^{(t+1)}}{\| \bm{R}_i^{(t+1)} 
\bm{w}_{i}^{(t+1)} \|_2} \nonumber \\
&& - h(\bm{w}_{i}^{(t)}) + ({}_{\nabla} \bm{h}^{(t)})^T \bm{w}_{i}^{(t)} \nonumber \\
&\leq& C \min_{\bm w} \left \{g(\bm{w}) - ({}_{\nabla} \bm{h}^{(t)})^T \bm{w} \right\} - h(\bm{w}_{i}^{(t)}) 
+ ({}_{\nabla} \bm{h}^{(t)})^T \bm{w}_{i}^{(t)},
\end{eqnarray}
where the constant $C$ is determined by the upper bounds of the following terms, \\
$\frac{L}{2} \| \bm{R}_i^{(t+1)} \bm{w}_{i}^{(t+1)} \|_2$, $\alpha \frac{\| \bm{Q}^{\#} \bm{R}_i^{(t+1)}  \bm{w}_{i}^{(t+1)}\|_1}{\| \bm{R}_i^{(t+1)} \bm{w}_{i}^{(t+1)} \|_2}$, and $ - \frac{({}_{\nabla} \bm{h}^{(t)})^T \bm{R}_i^{(t+1)} \bm{w}_{i}^{(t+1)}}{\| \bm{R}_i^{(t+1)} \bm{w}_{i}^{(t+1)} \|_2}$. These three terms can be bounded as follows.  
\begin{eqnarray}
\frac{L}{2} \| \bm{R}_i^{(t+1)} \bm{w}_{i}^{(t+1)} \|_2
& \leq & 
\left( \sqrt{ \lambda_{\mathrm{max}} \left( (\bm{R}_i^{(t+1)})^T \bm{R}_i^{(t+1)} \right)} \right)
\frac{L}{2}  \| \bm{w}_{i}^{(t+1)} \|_2, \label{eq:upper bound of 1st term} 
\end{eqnarray}
\begin{eqnarray}
\alpha \frac{\| \bm{Q}^{\#} \bm{R}_i^{(t+1)} \bm{w}_{i}^{(t+1)}\|_1}{\| \bm{R}_i^{(t+1)} \bm{w}_{i}^{(t+1)} \|_2} 
&=& 
\alpha \frac{\| \bm{Q}^{\#} \bm{R}_i^{(t+1)} \bm{w}_{i}^{(t+1)}\|_1}
{\| \bm{R}_i^{(t+1)} \bm{w}_{i}^{(t+1)} \|_2  \| \bm{Q}^{\#} \bm{w}_{i}^{(t+1)} \|_1}
 \| \bm{Q}^{\#} \bm{w}_{i}^{(t+1)} \|_1 \nonumber \\ 
 &\leq&
\alpha \frac{\sqrt{N} \| \bm{Q}^{\#} \bm{R}_i^{(t+1)} \bm{w}_{i}^{(t+1)}\|_2}
{\| \bm{R}_i^{(t+1)} \bm{w}_{i}^{(t+1)} \|_2  \| \bm{Q}^{\#} \bm{w}_{i}^{(t+1)} \|_2}
 \| \bm{Q}^{\#} \bm{w}_{i}^{(t+1)} \|_1 \nonumber \\ 
& \leq & \!\!\!\! \left( \!\! \sqrt{ \frac{ {N \lambda_{\mathrm{max}} \left({\bm{Q}^{\#}}^T \bm{Q}^{\#}\right)}}{{\lambda_{\mathrm{min}}\left({\bm{Q}^{\#}}^T \bm{Q}^{\#}\right)}}} 
\frac{1}{ \| \bm{w}_{i}^{(t+1)} \|_2} \!\! \right) \! \alpha \| \bm{Q}^{\#} \bm{w}_{i}^{(t+1)} \|_1, 
\label{eq:upper bound of 2nd term} 
\end{eqnarray}
\begin{eqnarray}
- \frac{({}_{\nabla} \bm{h}^{(t)})^T  \bm{R}_i^{(t+1)} \bm{w}_{i}^{(t+1)}}{\| \bm{R}_i^{(t+1)} \bm{w}_{i}^{(t+1)} \|_2}
& \leq & \| - {}_{\nabla} \bm{h}^{(t)} \|_2 \nonumber \\
&=& \!\!\!\!\! \left( \left( \frac{  ( - {}_{\nabla} \bm{h}^{(t)})^T \bm{w}_{i}^{(t+1)} }{\| - {}_{\nabla} \bm{h}^{(t)} \|_2 \| \bm w_i^{(t+1)} \|_2} \right)^{-1}
\frac{1}{\| \bm w_i^{(t+1)} \|_2} \right) 
( - {}_{\nabla} \bm{h}^{(t)})^T \bm{w}_{i}^{(t+1)}. \nonumber \\
\label{eq:upper bound of 3rd term} 
\end{eqnarray}
Note that the factor $\frac{  ( - {}_{\nabla} \bm{h}^{(t)})^T \bm{w}_{i}^{(t+1)} }{\| - {}_{\nabla} \bm{h}^{(t)} \|_2 \| \bm w_i^{(t+1)} \|_2} $ on the right hand side of inequality (\ref{eq:upper bound of 3rd term}) is cosine similarity.
The inequality in (\ref{eq:upper bound of 1st term}) is from the upper bound of the general quadratic form. 
The first inequality in (\ref{eq:upper bound of 2nd term}) is derived by the relation between $\ell_1$-norm and $\ell_2$-norm, namely $\| \bm{x} \|_2 \leq \| \bm{x} \|_1 \leq \sqrt{N} \| \bm{x} \|_2$ for an $N$-dimensional vector $\bm x$, 
and the second inequality is from the upper and the lower bounds of the general quadratic form. 
The inequality in (\ref{eq:upper bound of 3rd term}) is Cauchy-Schwarz inequality.
By these results and from the definition of $\bm w_i^{(t+1)}$ in (\ref{eq:update next point}), 
the last inequality in (\ref{eq:convergence-orthogonalize}) is shown using the constant $C$
in (\ref{eq:defC}).
Then, if the condition (\ref{eq:conditionin1}) is satisfied, one has
\begin{eqnarray}
&& C \min_{\bm w} \left \{g(\bm{w}) - ({}_{\nabla} \bm{h}^{(t)})^T \bm{w} \right\} 
- h(\bm{w}_{i}^{(t)}) + ({}_{\nabla} \bm{h}^{(t)})^T \bm{w}_{i}^{(t)} \nonumber \\
&\leq& g(\bm{w}_{i}^{(t)}) - ({}_{\nabla} \bm{h}^{(t)})^T \bm{w}_{i}^{(t)} 
- h(\bm{w}_{i}^{(t)}) + ({}_{\nabla} \bm{h}^{(t)})^T \bm{w}_{i}^{(t)} \nonumber \\
&=& g(\bm{w}_{i}^{(t)}) - h(\bm{w}_{i}^{(t)}) \nonumber \\
&=& J_o(\bm{w}_{i}^{(t)}).
\end{eqnarray}
This means the property $J_o(\widehat{\bm{w}}_{i}^{(t+1)}) \leq J_o(\bm{w}_{i}^{(t)})$ is proved, 
which guarantees the convergence.
\end{proof}

\section{Numerical experiment}\label{sec:Numerical Experiment}

\subsection{Application to synthetic data}
For validity of our proposed method, we apply our method to synthetic data to evaluate its performance. In addition, 
we compare the performance of our method with widely-used FastICA algorithm. 
In our numerical experiments, we use scikit-learn library (version 1.2.2) for Python (version 3.10.11).

\subsubsection*{Experimental conditions and performance measure}
In this experiment, we apply our proposed method and FastICA to 4 synthetic data 
$\bm X^{(1)}, \bm X^{(2)}$, $\bm X^{(3)}$, and $\bm X^{(4)}$, which are generated as follows. In particular, 
in constructing $\bm X^{(3)}$ and $\bm X^{(4)}$, 
the values of parameters are chosen to model large-size real fMRI data. (See also Table \ref{tab:syndata}.)

\begin{table}
\centering
\caption{Summary of synthetic data}
\begin{tabular}{c||c|c|c|c}
data & dist. of $\bm S$ & $p$ & $\chi$ & dist. of $\bm E$ \\ \hline
$\bm X^{(1)}$ & 80\% uniform, 20\% Laplace & 10 & 0.8 &${\cal N} (0,1)$ \\
$\bm X^{(2)}$ & 80\% uniform, 20\% log-normal & 10 & 0.8 & ${\cal N} (0,1)$ \\ \hline
$\bm X^{(3)}$ & 80\% uniform, 20\% log-normal & $10^4$ & 0.999 & ${\cal N} (0,1)$ \\
$\bm X^{(4)}$ & 80\% uniform, 20\% log-normal & $10^4$ & 0.999 & ${\cal N} (0,2^2)$ \\ \hline  
\end{tabular}
\label{tab:syndata}
\end{table}

\begin{enumerate}
\item
First, the ground-truth signal source matrix $\bm{S}^{*} \in \mathbb{R}^{K \times N}$ is generated, 
whose elements are drawn from Laplace distribution, log-normal distribution, or uniform distribution defined by
\begin{eqnarray}\label{eq:Laplace distribution}
\mathcal{L}(s_{ij}^*;\mu, \sigma) &=& \frac{1}{2\sigma}\exp \left(-\frac{|s_{ij}^{*}-\mu|}{\sigma} \right), \\
{\rm Lognormal} (s_{ij}^*; \mu, \sigma) &=& \frac{1}{s_{ij}^* \sqrt{2 \pi \sigma^2}}
\exp \left( - \frac{(\log s_{ij}^* - \mu)^2}{2 \sigma^2} \right),
\\
\mathcal{U}(s_{ij}^*;c_1, c_2) &=& 
\begin{cases}
\displaystyle \frac{1}{c_2-c_1} & {\rm for}\ c_1 \le s^*_{ij} \le c_2, \\
0 & {\rm otherwise.} 
\end{cases} 
\end{eqnarray}
Note that these distributions are non-gaussian distributions. 
For constructing $\bm X^{(1)}$, 
20\% of the vectors in $\bm S^*$ have elements following Laplace distribution 
and 80\% have elements following uniform distribution.
Similarly, 20\% of the vectors in $\bm S^*$ have elements following log-normal distribution 
and 80\% have elements following uniform distribution for $\bm X^{(2)}, \bm X^{(3)}$, and $\bm X^{(4)}$.
For all synthetic data, the dimensions of $\bm S$ are set as $K=10, N=500$.

\item 
Next, the ground-truth mixing matrix $\bm{A}^{*} \in \mathbb{R}^{p \times K}$ is generated,
whose element is drawn from Bernoulli-gaussian distribution.
\begin{equation}
{\rm BG} (a_{ij}^*) = \chi \delta(a_{ij}^*) + (1-\chi) \frac{1}{\sqrt{2 \pi}} \exp \left( -\frac{{a_{ij}^*}^2}{2} \right).
\end{equation}
For $\bm X^{(1)}$ and $\bm X^{(2)}$, we set $\chi=0.8, p=10$, namely the size of the data is small.
In contrast, we set $p=10^4$ for $\bm X^{(3)}$ and $\bm X^{(4)}$ to model real fMRI data, 
because the voxel size of fMRI data often takes the order from $10^4$ to $10^5$. 
Moreover, in real fMRI data analysis as mentioned in section \ref{sec:realdata}, elements of brain activity above 3 standard deviations from the mean
can be interpreted as clear characteristic activations. 
Namely, the number of elements representing strong activations is typically of the order of $10^1$.
From this observation, we set $\chi=0.999$ for $\bm X^{(3)}$ and $\bm X^{(4)}$.

\item
The observed matrix $\bm X \in \mathbb{R}^{p \times N}$ is constructed from $\bm{S}^{*}$ and $\bm{A}^{*}$
by the matrix product $\bm{X} = \bm{A}^{*}\bm{S}^{*} + \bm E$, where $\bm E \in \mathbb{R}^{p \times N}$ is noise matrix with the elements following standard gaussian distribution ${\cal N}(0,1)$ for $\bm X^{(1)}, \bm X^{(2)}$, and
$\bm X^{(3)}$. For $\bm X^{(4)}$, the element of noise matrix $\bm E$ follows ${\cal N}(0,2^2)$, namely noisier case.
\end{enumerate}

After application of our proposed method to the synthetic matrix $\bm X$, the result is evaluated by the 
performance measures in the following.
\begin{itemize}
\item
\textbf{sparsity} \\
Sparsity of the estimated mixture matrix $\bm{A}$ is defined in equation (\ref{eq:Sparsity}).
\begin{equation}\label{eq:Sparsity}
\mathrm{Sparsity}(\bm{A}) = \frac{ \sum_{i=1}^{p} \sum_{j=1}^{K}\mathbb{I} (\widehat{a}_{ij} < \epsilon )}{pK},
\end{equation}
where $\widehat{a}_{ij}$ is the $(i,j)$-element of normalized $\bm{A}$ and $\mathbb{I}(.)$ is the indicator function. 
In our experiment, we set the infinitesimal threshold $\epsilon = 10^{-2}$.
Under experimental condition mentioned above, 
the sparsity of the ground-truth mixture matrix $\bm A^{*}$ is 0.8 after truncation of the elements. 

\item
\textbf{kurtosis} \\
For non-gaussianity, mean of absolute values of kurtosis (denoted by MAK) for the estimated signal matrix $\bm S$ 
is used,
\begin{equation}\label{eq:kurtosis}
\mathrm{MAK}(\bm{S}) = 
\frac{1}{K} \displaystyle \sum_{i=1}^{K}
\left| \frac{\frac{1}{N} \sum_{j=1}^N s_{ij}^4}{(\frac{1}{N} \sum_{j=1}^N s_{ij}^2)^2 }  - 3 \right|,
\end{equation}
which is zero for the matrix with the elements following gaussian distribution.

\item
\textbf{RMSE as the reconstruction error} \\
For validity as an MF solution, 
rooted mean squared error (RMSE) is defined between the original observed matrix $\bm X$ 
and the product of estimated matrices $\bm A \bm S$ for describing reconstruction error.
\begin{equation}\label{eq:RMSE_X}
\mathrm{RMSE}_{\bm{X}} = \sqrt{\displaystyle \frac{1}{pN} \displaystyle \sum_{i=1}^{p} \displaystyle \sum_{j=1}^{N}\left(x_{ij} - \displaystyle \sum_{l=1}^{K}a_{il} s_{lj} \right)^{2}}.
\end{equation}

\item
\textbf{reconstruction success rate} \\ 
In general MF problem including ICA, the matrices after permutation of vectors 
in factorized matrix solution are also another solution. Therefore, it is not appropriate to measure RMSE between the ground-truth mixing matrix $\bm A^{*}$ and estimated mixing matrix $\bm A$ directly. 
To evaluate the difference between $\bm A$ and $\bm A^{*}$ by considering permutation symmetry of MF solution, Amari distance \cite{AmariDistance} 
is defined to measure how the product $\bm P = \bm W^T \bm Q \bm A^{*}$ is close to the permutation matrix as 
\begin{equation}\label{eq:Amari Distance}
\mathrm{AD}(\bm{P}) = \displaystyle \sum_{i=1}^{K} \left( \displaystyle \sum_{j=1}^{K} \displaystyle \frac{|p_{ij}|}{\underset{l}{\max} |p_{il}|}-1 \right) + \displaystyle \sum_{j=1}^{K} \left( \displaystyle \sum_{i=1}^{K} \displaystyle \frac{|p_{ij}|}{\underset{l}{\max} |p_{lj}|}-1 \right).
\end{equation}
Amari distance is regarded as a significant performance measure in ICA, because it is known that $\bm P$ is equivalent to the permutation matrix if signal sources can be separated perfectly.
 When an algorithm of ICA is executed in $n$ trials under given observed matrix $\bm X$ and different 
 initial matrix $\bm W$, the reconstruction success rate is defined using Amari distance,
\begin{equation}
{\rm SR} = \frac{ \sum_{j=1}^n \mathbb{I} [ {\rm AD} (\bm P^{(j)}) < \psi] }{n},
\end{equation}
where the superscript $(j)$ means the $j$-th trial.
\end{itemize}

In our experiment, these quantities are evaluated under various $\alpha$ and the sparsity of $\bm Q^{\#}$. For simplicity, the sparsity of $\bm Q^{\#}$ is denoted by $\kappa = 1-k_0/K$ in the following.
The parameters other than $\alpha$ and $\kappa$ are fixed as follows: 
$\mu = 0, \sigma=1, c_1=0, c_2=1, M=500, \rho=1, \zeta = 10^{-5}, \eta = 10^{-3}$. 
Sparsity, kurtosis, and RMSE are evaluated by our algorithm in 10 trials under the same observed matrix $\bm X$ and different initial matrix $\bm W$. Statistical error of multiple trials is shown by the error bar in the figures. 

\subsubsection*{Result 1: General property of factorized matrices by the proposed method}
We depict the behaviours of the sparsity of the estimated mixture matrix $\bm A$, the kurtosis of the estimated signal matrix $\bm S$, and the reconstruction error obtained by our proposed method and FastICA in Figure \ref{fig:figure1}.
The results for $\bm X^{(1)}$ (top, under the mixtures of Laplace/uniform distributions) and $\bm X^{(2)}$ 
(bottom, log-normal/uniform distributions) 
are compared. From this figure, there is no significant difference between mixtures of Laplace/uniform distributions and log-normal/uniform distributions. This implies that the result does not depend on 
the choice of non-gaussian distribution.

\begin{figure}[htbp]
	 \begin{center}
      \includegraphics[scale=0.4]{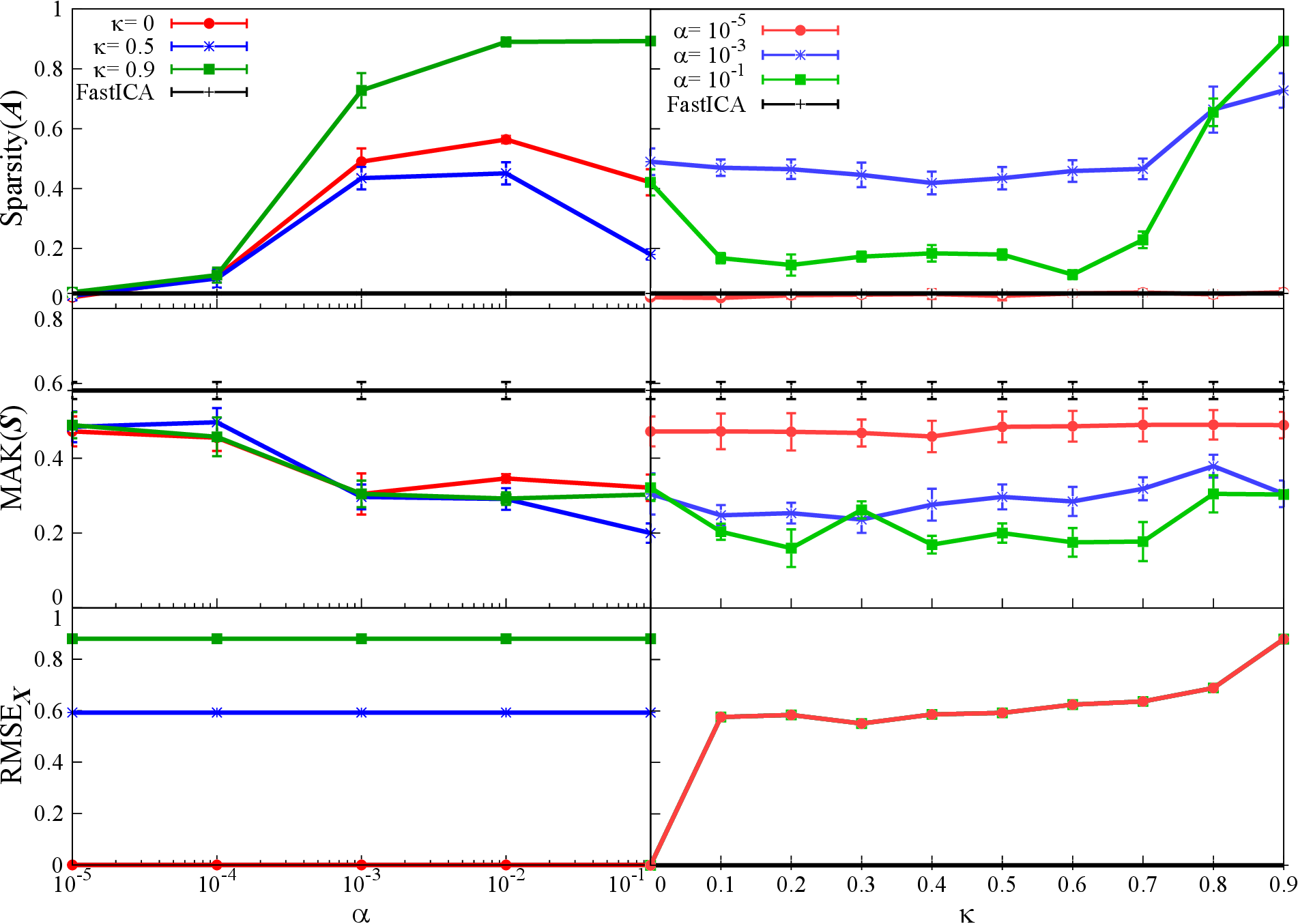}\vspace{5mm}
      \includegraphics[scale=0.4]{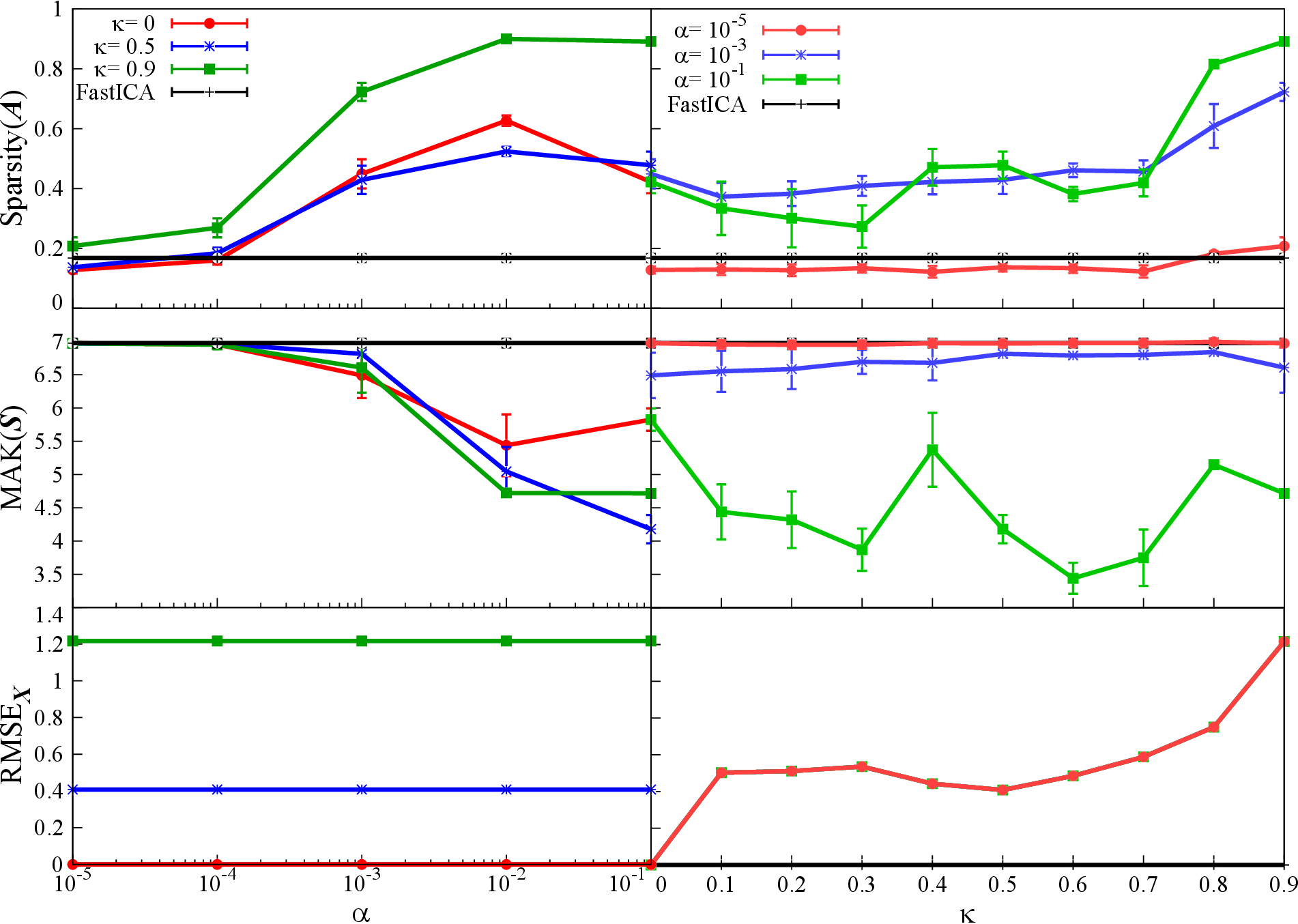}\vspace{5mm}
	  \caption{The behaviors of Sparsity($\bm{A}$), $\mathrm{MAK}(\bm{S})$, and  $\mathrm{RMSE}_{\bm{X}}$ versus $\alpha$ (left column) or $\kappa$ (right column) for $\bm X^{(1)}$ (top, mixture of Laplace/uniform) and $\bm X^{(2)}$ (bottom, mixture of log-normal/uniform): In both cases of $\bm X^{(1)}$ and $\bm X^{(2)}$,
	  the red ($\kappa=0$) and the black lines (FastICA) are overlapped on the horizontal axis in the bottom left figure, and all lines by our method ($\alpha=10^{-5}, 10^{-3}, 10^{-1}$) are overlapped in the bottom right figure.
	   }
	  \label{fig:figure1}
	 \end{center}
\end{figure}

We move on to the detail of the results in Figure \ref{fig:figure1}.
First, we discuss the sparsity of the estimated mixture matrix $\bm{A}$. The sparsity of $\bm{A}$ 
by our proposed method is larger than FastICA regardless of the value of $\alpha, \kappa$.
We also observe that the sparsity by our method 
is very large and close to 0.9 under larger $\alpha$ and $\kappa=0.9$.
When the value of $\alpha$ is fixed, large sparsity is obtained 
for large $\kappa$ under $\alpha = 10^{-2}$ or $10^1$.
In general, $\kappa$ must be larger for larger sparsity of $\bm{A}$, while $\alpha$ should be tuned carefully.

Next, we observe the behavior of ${\rm MAK}$.
${\rm MAK}$ by our method tends to decrease for larger $\alpha$, while $\rm MAK$ by FastICA does not. 
This may be because our ICA solution under larger $\alpha$ is different from the solution of the conventional FastICA. On the other hand, there is no significant correlation between $\kappa$ and ${\rm MAK}$. Therefore, for larger non-gaussianity, $\kappa$ can take arbitrary value, while $\alpha$ should be appropriately adjusted. 

We also mention the result of the reconstruction error. In Figure \ref{fig:figure1}, no change in 
${\rm RMSE}_{\bm X}$ is observed even if $\alpha$ is varied, and ${\rm RMSE}_{\bm X}$  depends only on $\kappa$. 
It is easy to prove analytically that ${\rm RMSE}_{\bm X}$  depends only on the approximation accuracy of $\bm Q^{\#}$ using the orthogonality of the matrix $\bm W$, which is also indicated by numerical result. In addition, 
by fixing $\alpha$, it is found that ${\rm RMSE}_{\bm X}$ tends to increase for larger $\kappa$. To summarize, $\alpha$ can be chosen arbitrarily for better reconstruction error, while $\kappa$ must be carefully adjusted. Additionally, from the fact that larger sparsity of $\bm A$ can be obtained under larger $\kappa$, large sparsity and small reconstruction error have a trade-off relation.

Lastly, we give the result of the reconstruction success rate.
Before evaluating SR of our method, 
we evaluate Amari distance by FastICA in 20 trials with different initial matrix, 
whose average is $43.04 \pm (8.87 \times 10^{-3})$.
Based on this result, the threshold of success is set as $\psi = 35$ in our experiment,
which is the appropriate value because it is below the average of Amari distance by FastICA.
For reconstruction success rate, we evaluate it in 20 trials with different initial matrix for 
$\alpha = 1.0 \times 10^{-4}, 2.0 \times 10^{-4}, \ldots, 1.0 \times 10^{-3}$. 
Parameter $\kappa$ is set to 0.9 in all cases. 
In Figure \ref{fig:SR}, we show the frequency of samples and SR under various sparsity of estimated matrix $\bm A$. 
The largest SR is observed in the range of $0.75 \leq {\rm Sparsity}(\bm A) < 0.85$. 
As stated in the experimental conditions, the sparsity of the ground-truth mixture matrix $\bm A^{*}$ is 0.8.
Therefore, this result indicates that the ground-truth mixture matrix $\bm A^{*}$ can be obtained by our method 
with high probability, especially when the sparsity of $\bm A$ is close to the one of $\bm A^{*}$. 

\begin{figure}
	 \begin{center}
     \includegraphics[scale=0.8]{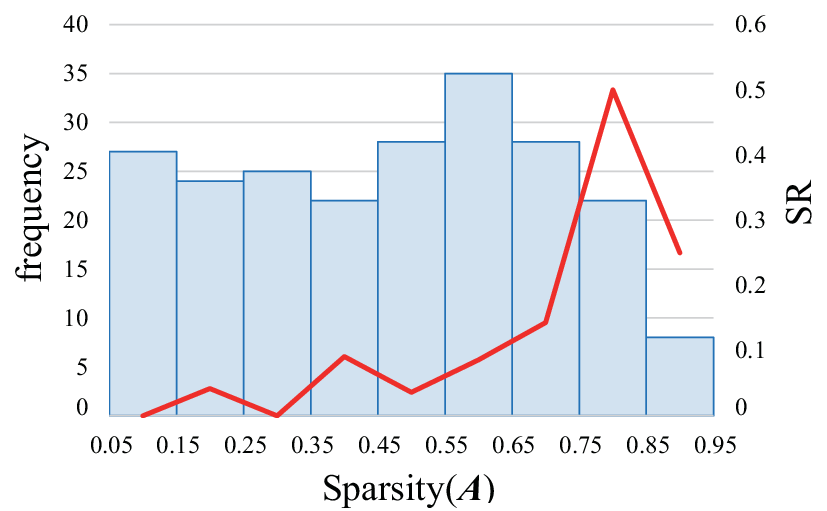}
	 \caption{Frequency and SR versus ${\rm Sparsity}(\bm{A})$:
	 The blue bar and red line indicate the frequency of ${\rm Sparsity} (\bm A)$ and SR at each bin, respectively. 10 samples with ${\rm Sparsity}(\bm A) < 0.05$ are not displayed in this figure.}
	 \label{fig:SR}
	 \end{center}
\end{figure}

\subsubsection*{Result 2: Application to large synthetic data modelling fMRI}

We also apply our method and FastICA to synthetic data $\bm X^{(3)}$ and $\bm X^{(4)}$, 
which model real fMRI data with larger dimension.
Similar to the applications to $\bm X^{(1)}$ and $\bm X^{(2)}$, the behaviors of the sparsity of $\bm A$, 
the kurtosis of $\bm S$, and the reconstruction error are depicted in Figure \ref{fig:figure3}. 

\begin{figure}[htbp]
	 \begin{center}
      \includegraphics[scale=0.4]{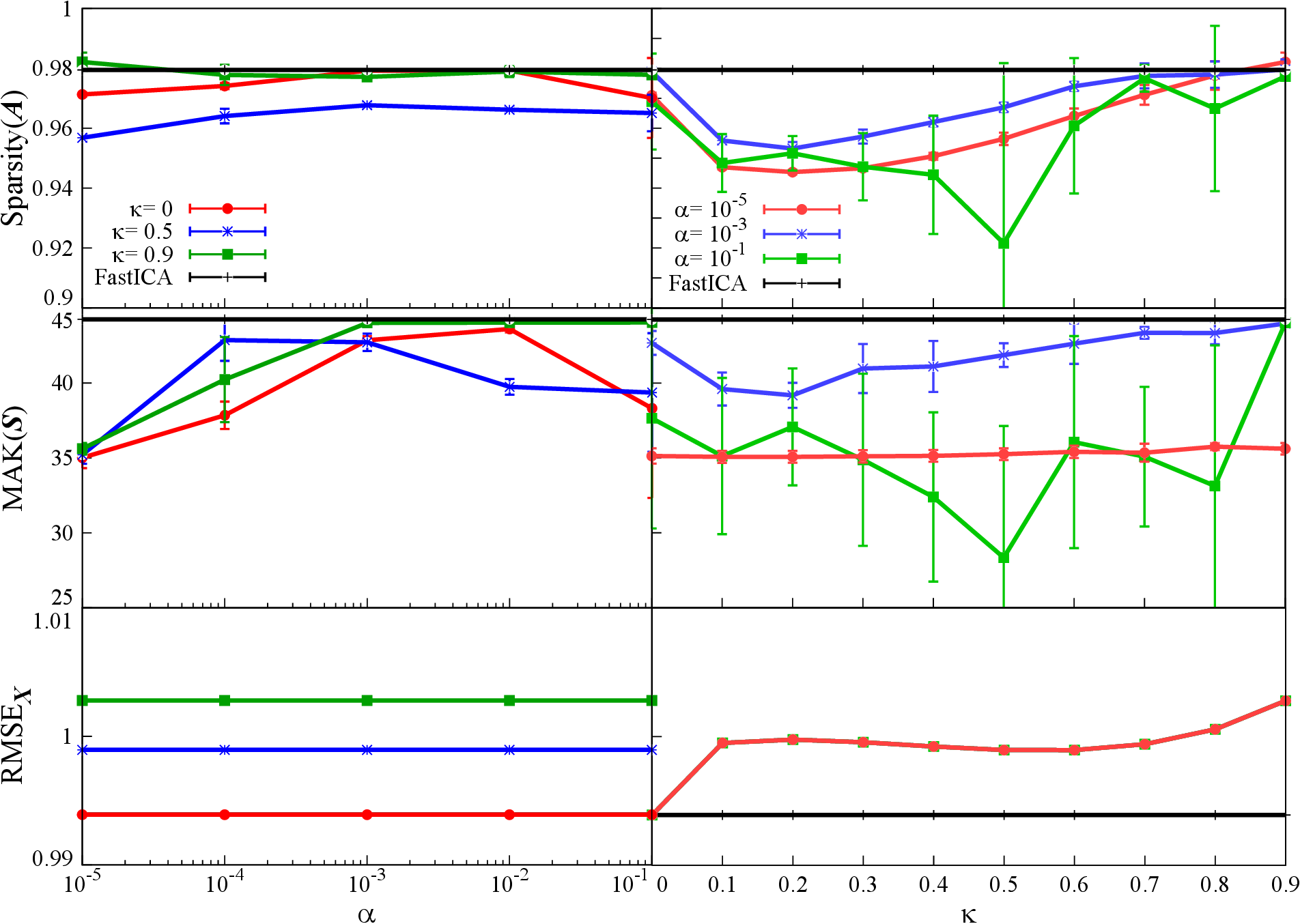}\vspace{5mm}
      \includegraphics[scale=0.4]{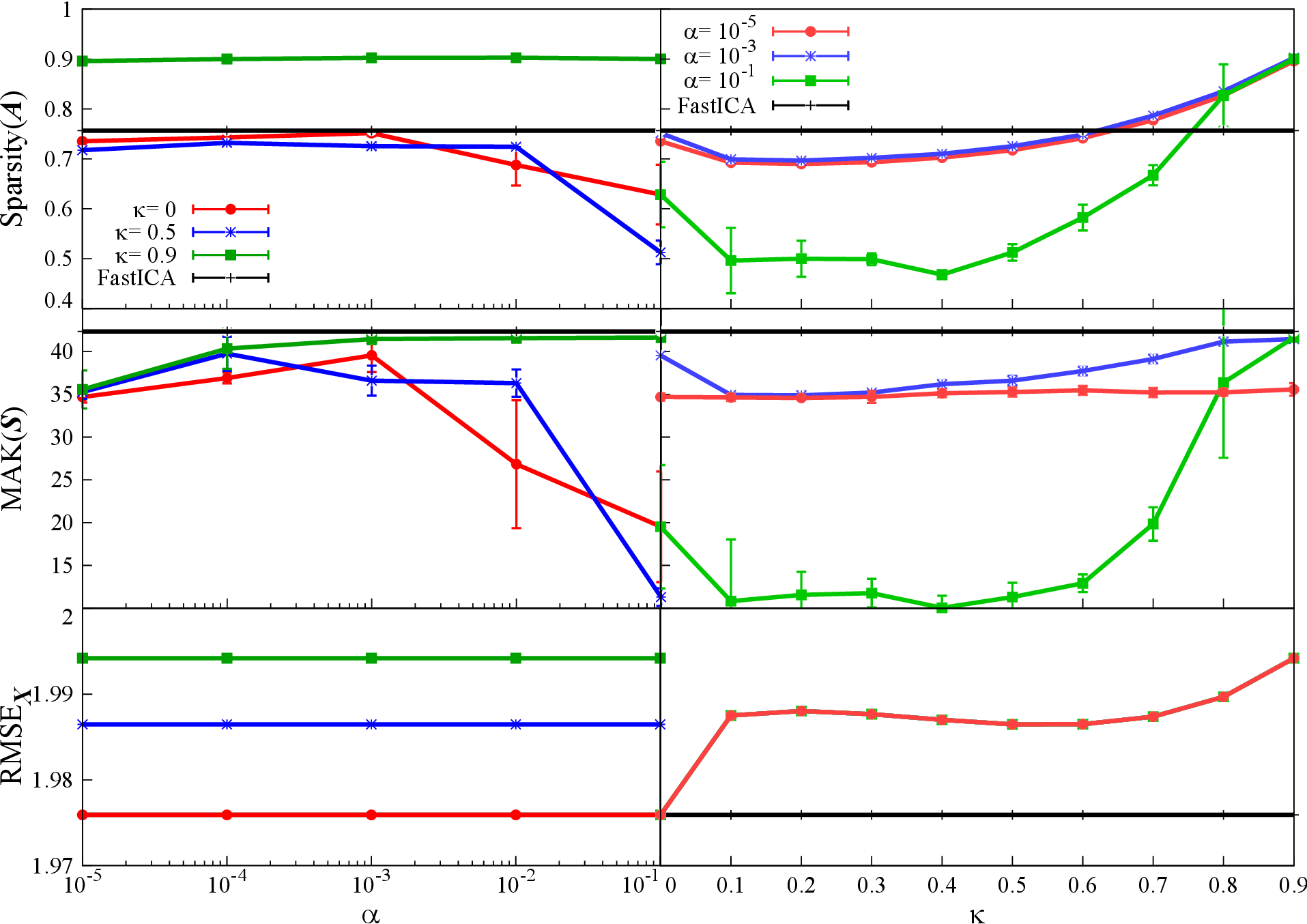}\vspace{5mm}
	  \caption{The behaviors of Sparsity($\bm{A}$), $\mathrm{MAK}(\bm{S})$, and  $\mathrm{RMSE}_{\bm{X}}$ versus $\alpha$ (left column) or $\kappa$ (right column) for large size data
	  $\bm X^{(3)}$ (top) and $\bm X^{(4)}$ (bottom, noisy): 
	  In both cases of $\bm X^{(3)}$ and $\bm X^{(4)}$, the red ($\kappa=0$) and the black lines (FastICA) are overlapped on the horizontal axis in the bottom left figure, and all lines by our method ($\alpha=10^{-5}, 10^{-3}, 10^{-1}$) are overlapped in the bottom right figure.
	   }
	  \label{fig:figure3}
	 \end{center}
\end{figure}

The result is summarized as follows.
First, sparsity of $\bm A$ is almost unchanged even when $\alpha$ is varied, especially in applying to $\bm X^{(3)}$. 
As for the dependence on $\kappa$, sparsity takes a large value at $\kappa=0$, and
it suddenly decrease then gradually tends to increase when $\kappa$ is increased from $0$.
In the application of $\ell_1$-regularized MF,
similar phenomenon of little change in sparsity when varying $\alpha$ has already been observed 
in the previous study \cite{KT2023}: if synthetic data $\bm X$ is generated
by very sparse ground-truth factorized matrix $\bm A^*$, application of SparsePCA does not significantly change 
the sparsity of the resulting factorized matrix even when the sparse parameter in SparsePCA is varied.
Since the evaluation of sparsity here is performed similarly to SparsePCA
in the previous study, we guess that similar behavior of sparsity is obtained in this experiment.
In the application to $\bm X^{(4)}$,
the change of sparsity depends on $\kappa$ rather than $\alpha$, which is similar to $\bm X^{(3)}$. 
Note that the sparsity by FastICA is not so large in this case, while our method can sparsify factorized matrix.
This indicates the validity of our method to obtain sparse factorized matrix from noisy large-size data.

Next, for $\bm X^{(3)}$, there is no clear trend between kurtosis of $\bm S$ and both of sparse parameters $\alpha, \kappa$,
while the statistical error of kurtosis is very large under $\alpha=10^{-1}$ excepting the point $\kappa=0.9$. 
This result suggests that parameters should be tuned carefully for the stable solution with large kurtosis,
if the data is less noisy. In contrast, for $\bm X^{(4)}$, 
kurtosis tends to increase from $\kappa=0.6$ to 0.9.
From this result, $\kappa$ should be larger for larger kurtosis, when our method is applied to noisy large-size data.

The behavior of reconstruction error is similar to Figure \ref{fig:figure1}, while its value is
much larger than those in Figure \ref{fig:figure1}. 
This is caused by the fact that Moore-Penrose pseudo inverse matrix cannot approximate 
the true inverse matrix appropriately under $p>N$, which is the problem inherent in original ICA.

Lastly, we compare the result of application to $\bm X^{(4)}$ with FastICA. 
Before comparison with FastICA, 
we evaluate the sparsity and AD by FastICA in 20 trials with different initial matrix. 
The average of sparsity is $0.765 \pm (2.85\times 10^{-4})$ 
and the average of AD is $38.82 \pm 1.17$. 
Then, in the application of our method here,
the parameter $\kappa$ is set to be $0.7, 0.8, 0.9$, and the experiment is conducted in 20 trials with different initial matrix
for each value of $\kappa$ (totally 60 trials). 
The parameter $\alpha$ is fixed at $10^{-1}$. 
Table \ref{tab:noisyX4} shows the result by our method. The threshold in the definition of SR is set 
at $\psi = 38.82$ for comparison with FastICA.
This result in the table suggests that 
the factorized matrix with lower AD than FastICA can be obtained
with approximately 50\% probability, when the sparsity is larger than FastICA.
Namely, our proposed method 
can obtain a closer factorized matrix to the ground-truth mixing matrix $\bm S^*$ 
than FastICA at a certain probability. This suggests that our method is competitive 
with original FastICA for evaluating of sparse and independent factorized matrices
from large-size noisy data.
As written in section \ref{sec:realdata}, 
it should be emphasized that the validity of our method for application to fMRI data will be discussed.

\begin{table}
\centering
\caption{Result of application to $\bm X^{(4)}$ (noisy data)}
\begin{tabular}{c||c|c|}
sparsity & frequency & SR \\ \hline
smaller than FastICA &  19 & 0.211 \\ \hline
larger than FastICA & 41 & 0.488 \\ \hline
\end{tabular}
\label{tab:noisyX4}
\end{table}

\subsection{Application to real-world data}
\label{sec:realdata}
Next, for verifying the practical utility of our method, we conduct an experiment for feature extraction in real fMRI data. 

\subsubsection*{Dataset and performance measure}
We apply our method to Haxby dataset \cite{Haxbydata}, which is a task-related fMRI dataset recording human's response to 8 images. The 8 images are as follows: shoe, house, scissors, scrambledpix, face, bottle, cat, and chair. 
In the experiment of fMRI data acquisition, one image is shown to a test subject for a while, then the next image is shown after an interval.
Finally, all 8 images are shown sequentially.
 The whole Haxby dataset includes fMRI data of 12 trials from 6 test subjects. 
The previous study \cite{ET2024} shows no significant difference among subjects in this dataset. 
Therefore, in our experiment, we apply our method and FastICA to one trial data 
of test subject No. 2, which was acquired from 39912 voxels with 121 scanning time steps.
 In our experiment, we apply our method and FastICA to one trial data from the test subject No. 2, 
 which was acquired from 39912 voxels with 121 scanning time steps.
 Namely, the data matrix size is $39912 \times 121$ or $\bm{X} \in \mathbb{R}^{39912 \times 121}$. 

As performance measures, we use sparsity in equation (\ref{eq:Sparsity})
and correlation defined by
\begin{equation}\label{eq:correlation}
\mathrm{Correlation}(\bm{s}_i, {\bm{d}^{\mathrm{REST}}}) = \displaystyle \frac{\left| (\bm{s}_i - \bar{\bm{s}_i})^T (\bm{d}^{\mathrm{REST}} - \overline{\bm{d}^{\mathrm{REST}}}) \right|}{\| \bm{s}_i - \bar{\bm{s}_i} \|_2 \| \bm{d}^{\mathrm{REST}} - \overline{\bm{d}^{\mathrm{REST}}} \|_2 },
\end{equation}
where overline means the arithmetic average of all elements in the vector. 
The vector $\bm{s}_i$ means the $i$-th temporal feature vector in the matrix $\bm S$.
The column vector $\bm{d}^{\mathrm{REST}}$ represents the ground-truth timing of resting state: 
the $j$-th element of $\bm{d}^{\mathrm{REST}}$ is 1 if no image is shown to the test subject at the $j$-th time step, 
and 0 if one of the images is shown, respectively. Therefore, the quantity in equation (\ref{eq:correlation}) 
measures the correlation between the ground-truth timing vector of the resting state $\bm d^{\rm REST}$ 
and temporal feature vector in $\bm s$.

\subsubsection*{Result}

\begin{figure}[htbp]
	 \begin{center}
      \includegraphics[scale=0.7]{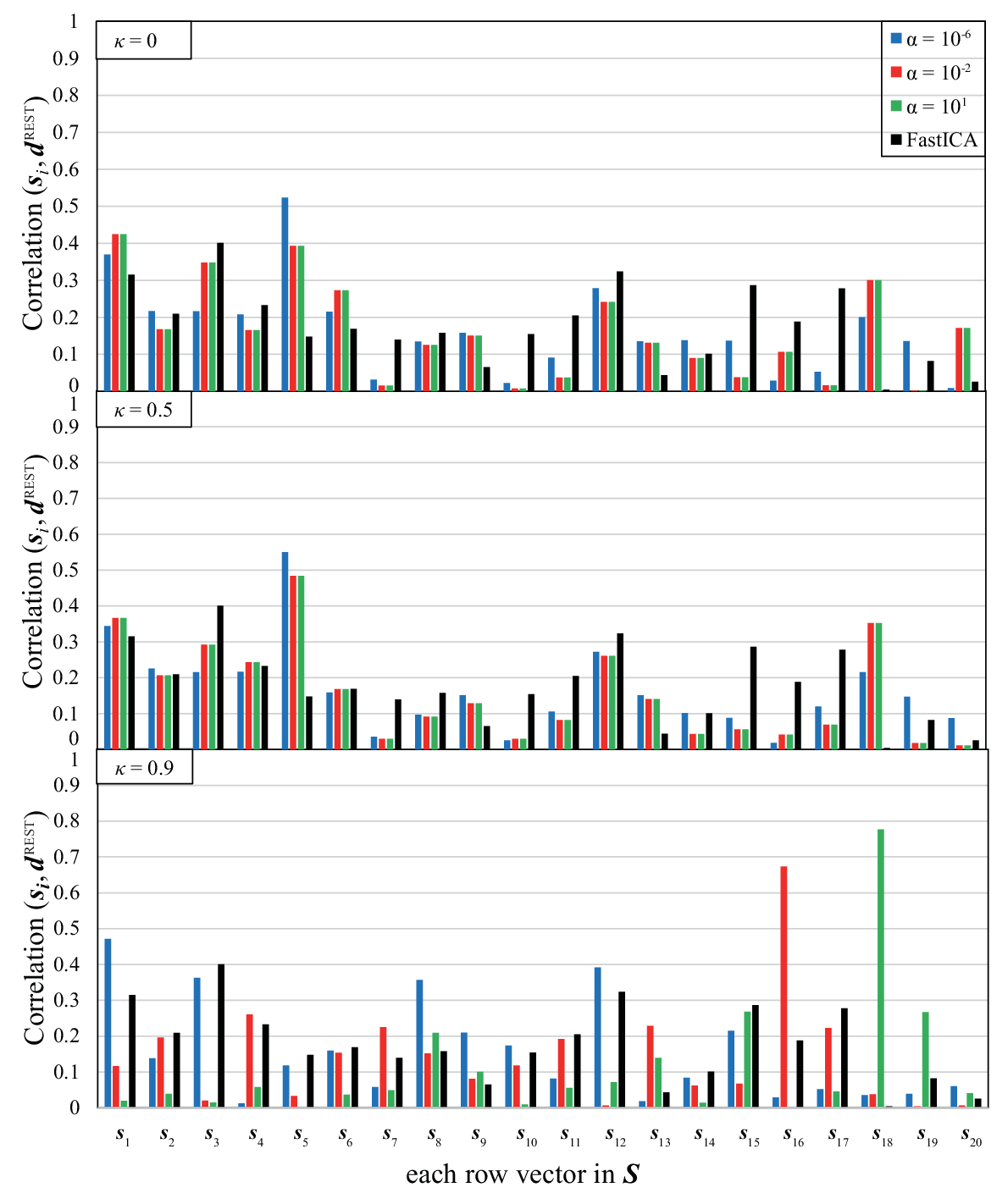}
	  \caption{Correlation by our method: The values of Correlation between the respective temporal feature vector 
	  in $\bm{S}$ and the ground-truth timing vector $\bm{d}^{\mathrm{REST}}$ under various $\alpha, \kappa$ are shown.}
	  \label{fig:correlation}
	 \end{center}
\end{figure}

The result of correlation is depicted in Figure \ref{fig:correlation} under various $\alpha, \kappa$: $\alpha = 10^{-6}, 10^{-2}, 10^1$ and $\kappa = 0, 0.5, 0.9$. 
Other parameters are set as follows: $K=20, M=500, \rho=1, \zeta=10^{-5}, \eta=10^{-3}$. 
From this figure, the value of correlation is at most 0.6 for $\kappa=0, 0.5$, 
whereas it sometimes exceeds 0.6 under $\alpha=10^{-2}, 10^1$ for $\kappa=0.9$. 
In addition, to confirm the validity of our method, 
we apply FastICA and our proposed method under $\kappa=0.9,\alpha=10^1$ to the same data 50 times with different initial matrix,
then conduct Student's two-sample $t$-test for the first, the second, and the third largest values of Correlation.
The result is shown in Table 3, where negative value means that the Correlation by our method is larger.
From this result, there is a significant difference between the results by our method and FastICA, and
it is clear that the first and the second largest values of Correlation by our method are larger than FastICA.
Although the third largest value by FastICA is larger,
we think it is sufficient to claim the advantage of our method over FastICA. 
 
\begin{table}
\centering
\caption{Result of Student's two-sample $t$-test}
\begin{tabular}{c||c|c|c|}
value of Correlation  & 1st & 2nd & 3rd \\ \hline
t-statistic &  -25.7 & -3.33 & 5.43 \\ \hline
p-value & $4.31 \times 10^{-30}$ & $1.52 \times 10^{-3}$ & $1.14 \times 10^{-6}$ \\ \hline
\end{tabular}
\label{tab:tstatistic}
\end{table}

Next, we visualize the extracted feature vectors by our method. 
The 18th row vector in $\bm S$ by our method under $\alpha=10^1, \kappa=0.9$ 
is depicted in Figure \ref{fig:proposed18th}A. 
Note that the 18th row vector has the largest Correlation in our result
under $\alpha = 10^1, \kappa=0.9$ as shown in Figure \ref{fig:correlation}. 
The timing of showing image to a test subject, which reflects the information of the vector $\bm{d}^{\mathrm{REST}}$, is also shown in the figure. 
From this result, the temporal changes between resting and non-resting states can be tracked 
easily in the feature vector by our method. 
Additionally, the spatial map of the 18th spatial feature vector in $\bm A$ by our method 
is shown on the cross sections of the brain in Figure \ref{fig:proposed18th}B.
Note that this spatial feature vector is the counterpart of the 18th temporal feature vector 
in Figure \ref{fig:proposed18th}A. 
For spatial map, the column vectors in $\bm A$ are normalized to have zero mean and unit variance, and elements within 3 standard deviations of the mean are truncated to zero.
For visualization of the spatial map, we use Nilearn (version 0.10.1) in Python library. 
In this figure, several active regions under resting state or the response to some visual stimuli can be observed. 
In particular, strong activations are observed in the cerebellum, which is known to be activated by visual stimuli. 
This result indicates that our method with high sparsity setting can identify brain regions 
for information processing of visual stimuli with high accuracy.

To compare our proposed method and FastICA, 
we show the behaviors of two feature vectors given by FastICA, namely the 3rd
and the 17th vectors.
For clarity, extracted feature vectors are denoted with a superscript, which represents the method 
for comparing FastICA and our proposed method. For example, the 3rd vector in $\bm S$ by FastICA 
is written as $\bm s_3^{(\rm FastICA)}$.
First, the vector $\bm s_3^{(\rm FastICA)}$ and the spatial map of $\bm a_3^{(\rm FastICA)}$ 
are depicted in Figures \ref{fig:fastica3rd}A and \ref{fig:fastica3rd}B, respectively. 
Note that the vector $\bm s_3^{(\rm FastICA)}$ is most strongly correlated 
with the timing vector of visual stimuli as in Figure \ref{fig:correlation}. 
However, from Figure \ref{fig:fastica3rd}A, resting and non-resting regions in fMRI data cannot be discriminated 
from the shape of the feature vector by FastICA. 
In addition, we cannot clearly identify which parts of the brain are strongly activated from Figure \ref{fig:fastica3rd}B. 
We also depict the vector $\bm s_{17}^{\rm (FastICA)}$ and the spatial map of $\bm a_{17}^{\rm (FastICA)}$ 
in Figures \ref{fig:fastica17th}A and \ref{fig:fastica17th}B, respectively.
Note that the vector $\bm a_{17}^{\rm (FastICA)}$ is 
most strongly correlated with the vector $\bm a_{18}^{\rm (ours)}$, 
whose counterpart $\bm s_{18}^{\rm (ours)}$ has the largest value of Correlation 
with $\bm d^{\rm REST}$ under appropriate $\alpha, \kappa$ as in Figure 4.
For Correlation between these two column vectors, see also Figure \ref{fig:comparison}A in the following.
Similarly to the vector $\bm s_3^{(\rm FastICA)}$, 
we cannot find significant synchronization with the timing of visual stimuli 
from Figure \ref{fig:fastica17th}A, and it is difficult to understand what the spatial map of $\bm a_{17}^{\rm (FastICA)}$ means because the area of activation in the brain is not clear from Figure \ref{fig:fastica17th}B. 

For relation between spatial maps by the two methods, we evaluate the absolute value of Correlation 
between the vector $\bm a_{18}^{\rm (ours)}$ and each column vector of $\bm A$ by FastICA in Figure \ref{fig:comparison}A. 
From this result, we find that the vector $\bm a_{17}^{\rm (FastICA)}$ 
is most strongly correlated with the vector $\bm a_{18}^{\rm (ours)}$. 
In Figure \ref{fig:comparison}B, 
we depict the scatter plot of the element in the vector $\bm a_{18}^{\rm (ours)}$ vs. 
the corresponding element in the vector $\bm a_{3}^{\rm (FastICA)}$ (top) or $\bm a_{17}^{\rm (FastICA)}$ (bottom), where the values of the elements in spatial vectors are normalized to the range $[-1, 1]$ by the linear transformation. 
From this figure, the vector $\bm a_{17}^{\rm (FastICA)}$ is more strongly correlated with 
the vector $\bm a_{18}^{\rm (ours)}$. 
Hence, these two spatial feature vectors will represent similar spatial networks. 
As mentioned in the result of Figure \ref{fig:fastica17th}, 
synchronization with the timing of visual stimuli is not observed in the vector $\bm s_{17}^{\rm (FastICA)}$, 
which is the counterpart of the spatial feature vector $\bm a_{17}^{\rm (FastICA)}$.
In contrast, the temporal feature vector $\bm s_{18}^{\rm (ours)}$ is clearly synchronized 
with the timing of visual stimuli, and the corresponding spatial feature $\bm a_{18}^{\rm (ours)}$ also shows activation in the region related to visual stimuli.
From these facts, we can conclude that our method outperforms FastICA in feature extraction from fMRI data.

\begin{figure}[htbp]
	 \begin{center}
      \includegraphics[scale=0.9]{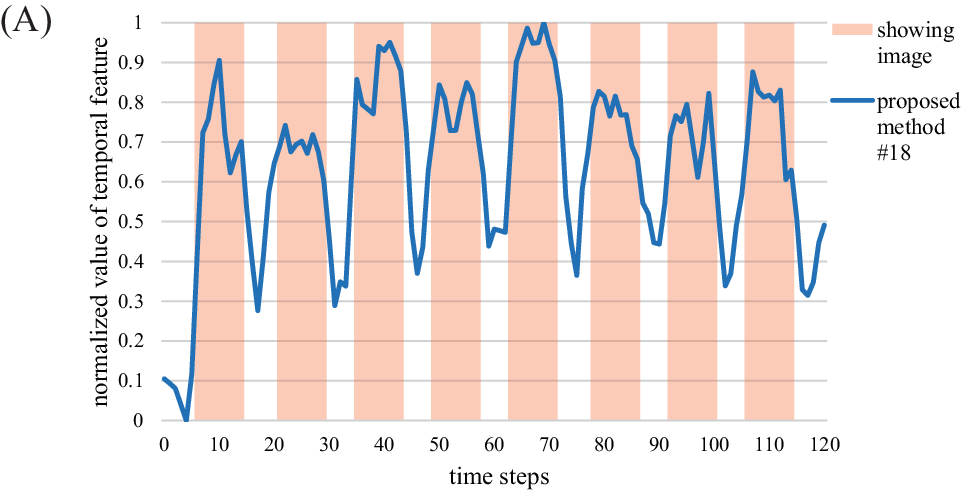}\vspace{1cm}
      \includegraphics[scale=0.7]{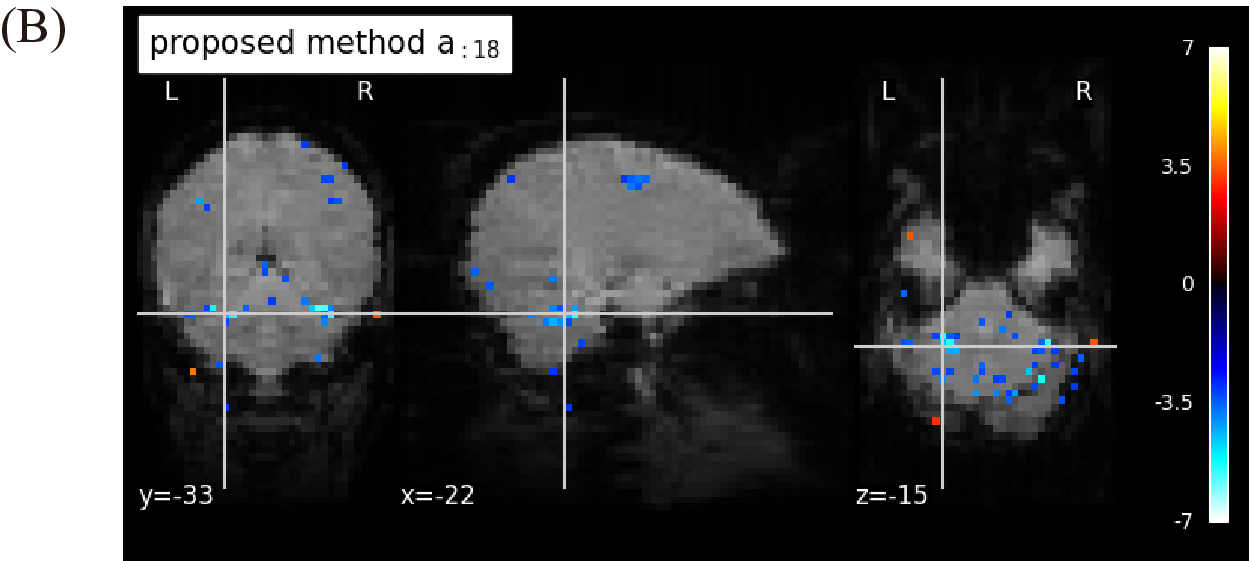}\vspace{1cm}
      \caption{(A) The 18th row vector in $\bm S$ by our proposed method, $\bm s_{18}^{\rm (ours)}$: 
      The orange background indicates the timing of showing image to a test subject. 
      (B) Spatial map depicted on the cross-section of the brain: 
      Each point has one-to-one correspondence 
      with the element in the 18th spatial feature vector in $\bm A$, $\bm a_{18}^{\rm (ours)}$. 
      The points are displayed according to the mapping between voxel position and each element 
      in the spatial feature vector. The color represents the value of the element.}
	  \label{fig:proposed18th}
	 \end{center}
\end{figure}

\begin{figure}[htbp]
	 \begin{center}
      \includegraphics[scale=0.9]{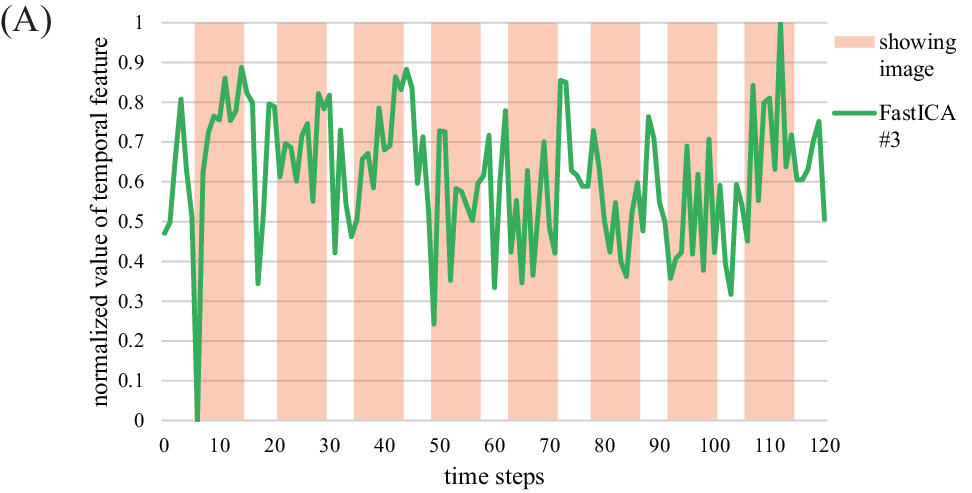}\vspace{1cm}
      \includegraphics[scale=0.7]{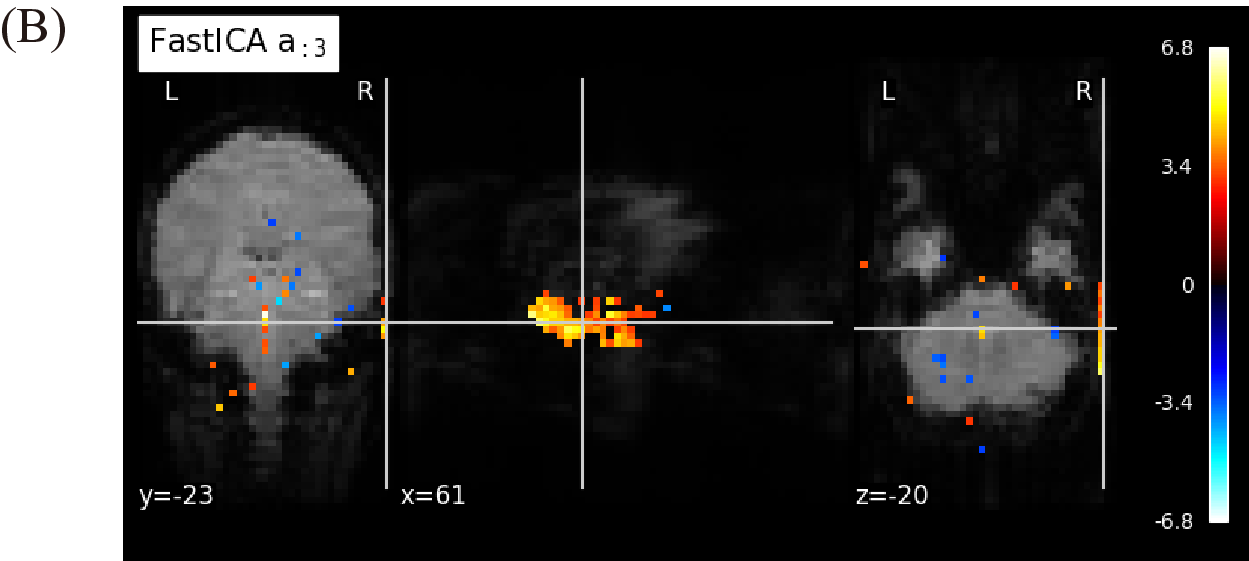}\vspace{1cm}
      \caption{(A) The vector $\bm s_{3}^{\rm (FastICA)}$: 
      The orange background indicates the timing of showing image to a test subject. 
      (B) Spatial map depicted on the cross-section of the brain: 
      Each point has one-to-one correspondence with the element in the vector 
      $\bm a_{3}^{\rm (FastICA)}$.}
	  \label{fig:fastica3rd}
	 \end{center}
\end{figure}

\begin{figure}[htbp]
	 \begin{center}
      \includegraphics[scale=0.9]{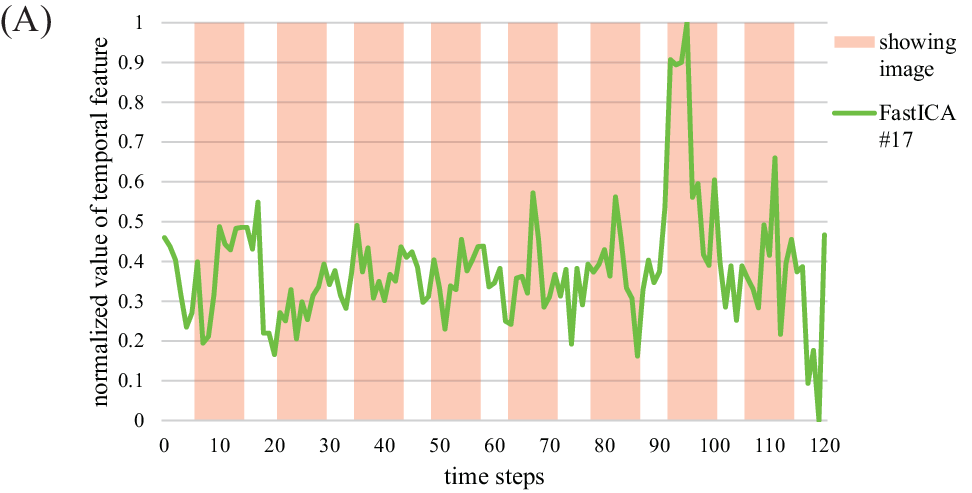}\vspace{1cm}
      \includegraphics[scale=0.7]{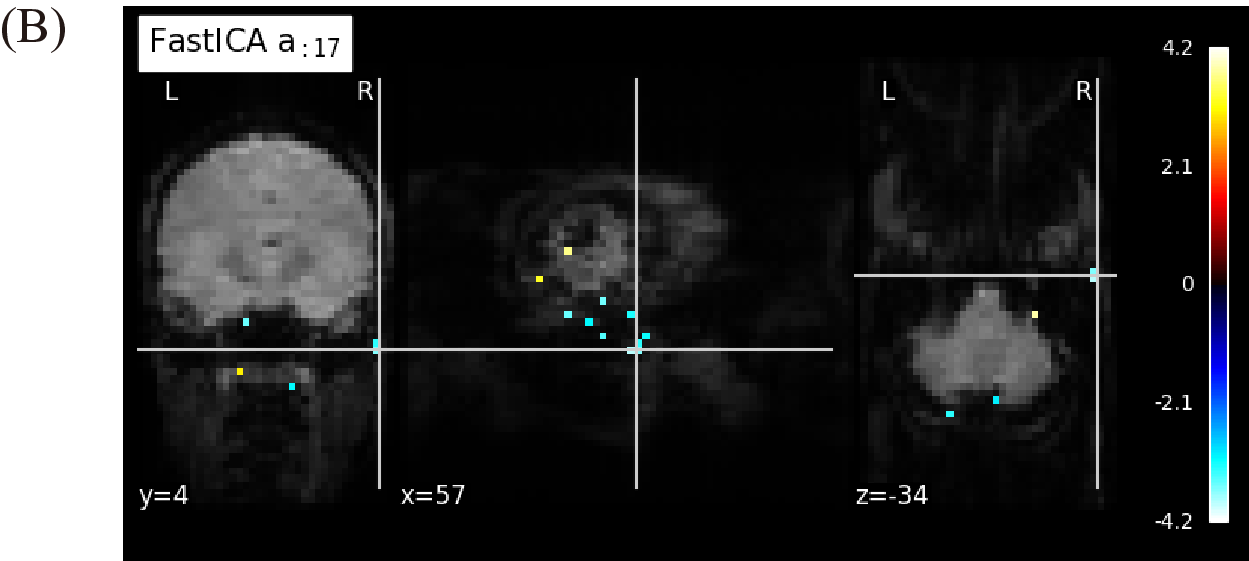}\vspace{1cm}
      \caption{(A) The vector $\bm s_{17}^{\rm (FastICA)}$: 
      The orange background indicates the timing of showing image to a test subject. 
      (B) Spatial map depicted on the cross-section of the brain: 
      Each point has one-to-one correspondence with the element 
      in the vector $\bm a_{17}^{\rm (FastICA)}$.}
	  \label{fig:fastica17th}
	 \end{center}
\end{figure}

\begin{figure}[htbp]
	 \begin{center}
      \includegraphics[scale=0.7]{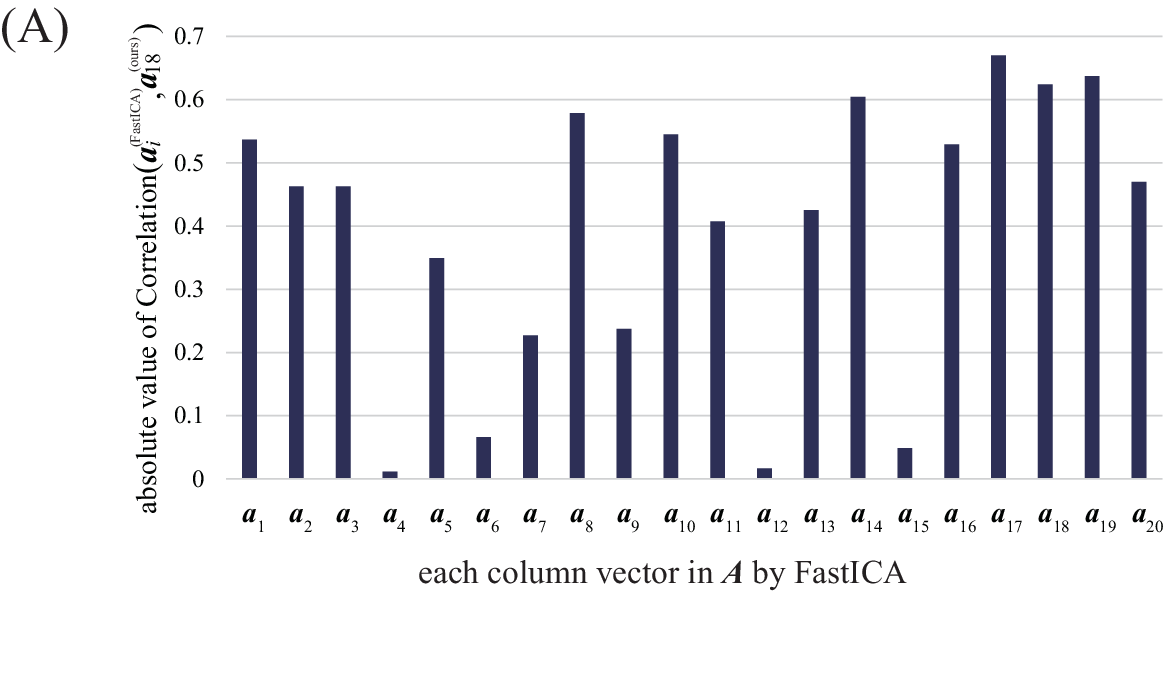}
      \includegraphics[scale=0.7]{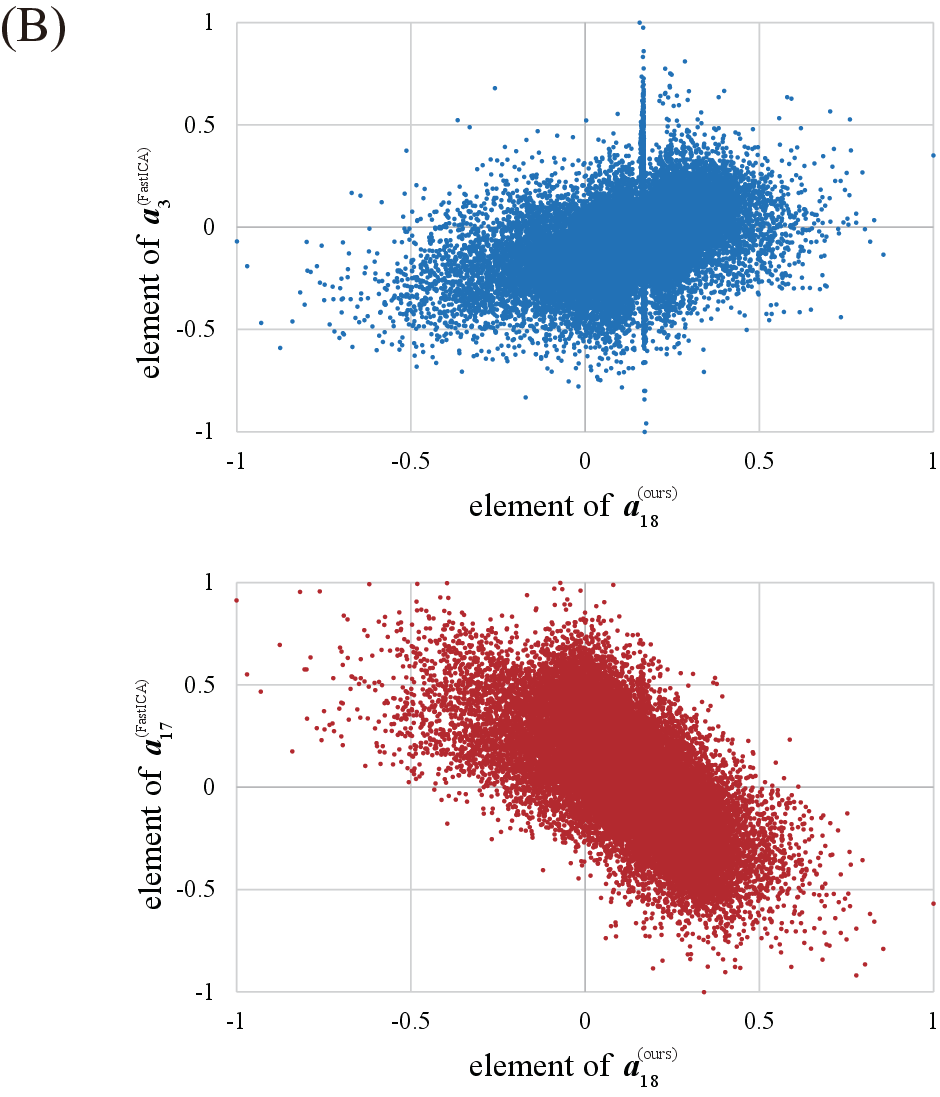}
      \caption{(A) Correlation between the vector $\bm a_{18}^{\rm (ours)}$ 
      and each column vector in $\bm A$ by FastICA 
      (B) Scatter plot of the element in the vector $\bm a_{3}^{\rm (FastICA)}$ vs. 
      the corresponding element in the vector $\bm a_{18}^{\rm (ours)}$ (top), 
      similarly the vector $\bm a_{17}^{\rm (FastICA)}$ vs. the vector $\bm a_{18}^{\rm (ours)}$ (bottom)}
	  \label{fig:comparison}
	 \end{center}
\end{figure}

We think the advantage of our method is due to the sparsity of the estimated mixture matrix $\bm{A}$.
To support this, the sparsity of $\bm{A}$ by this experiment is summarized in Table \ref{tab:SKR}.
From this table, it is found that the parameters giving the feature vector with the largest correlation 
($\alpha=10^1, \kappa=0.9$) 
lead to the sparsest matrix $\bm{A}$.
On the other hand, ${\rm Sparsity}(\bm{A})$ by FastICA is evaluated as 0.104, 
which is much smaller than our method.
In the previous study \cite{sMF2fMRIexA, sMF2fMRIexB, ET2024}, 
it is claimed that the method of MF giving sparse $\bm{A}$ 
can extract appropriate temporal feature vector characterizing the external stimuli. 
Therefore, the result indicating the advantage of our method is consistent with the previous studies.

\begin{table}
	 \centering
	 \caption{$\rm Sparsity$($\bm{A}$), $\rm MAK$($\bm S)$, and $\mathrm{RMSE}_{\bm X}$ 
	 under various $\alpha, \kappa$}
	 \begin{tabular}{l|rrr}
${\rm Sparsity}(\bm A)$ & $\kappa=0$ & $\kappa=0.5$ & $\kappa=0.9$ \\ \hline
$\alpha=10^{-6}$ & 0.095  & 0.099  & 0.099  \\ \hline
$\alpha=10^{-2}$ & 0.098  & 0.101  & 0.101  \\ \hline
$\alpha=10^{1}$ & 0.156  & 0.505  & 0.770  \\ \hline
	 \end{tabular}\vspace{3mm} \\
	 \begin{tabular}{l|rrr}
 ${\rm MAK}(\bm S)$ & $\kappa=0$ & $\kappa=0.5$ & $\kappa=0.9$ \\ \hline
$\alpha=10^{-6}$ & 0.91  & 0.91  & 2.08  \\ \hline
$\alpha=10^{-2}$ & 1.27 & 1.41  & 1.36  \\ \hline
$\alpha=10^{1}$ & 1.27  & 1.41  & 1.50  \\ \hline
	 \end{tabular}\vspace{3mm} \\
	 \begin{tabular}{l|rrr}
 ${\rm RMSE}_{\bm X}$ & $\kappa=0$ & $\kappa=0.5$ & $\kappa=0.9$ \\ \hline
$\alpha=10^{-6}$ & 0.904  & 0.908  & 0.922  \\ \hline
$\alpha=10^{-2}$ & 0.904  & 0.908  & 0.922  \\ \hline
$\alpha=10^{1}$ & 0.904 & 0.908  & 0.922  \\ \hline
	 \end{tabular}
\label{tab:SKR}
\end{table}

\begin{figure}[htbp]
	 \begin{center}
      \includegraphics[scale=0.85]{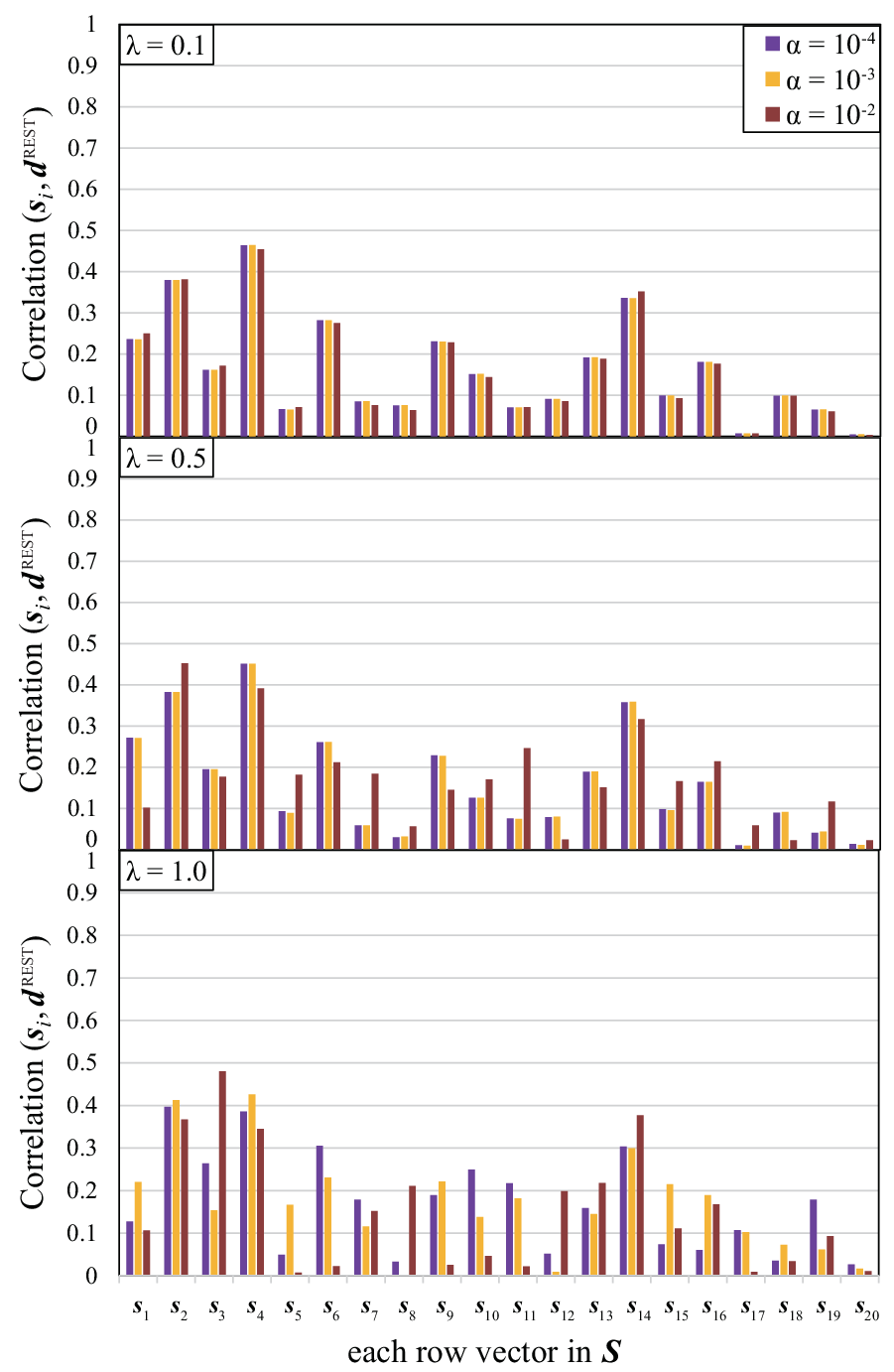}
	  \caption{Correlation by another sparse ICA method in the previous work \cite{harada2021sparse}: 
	  The values of Correlation between the respective temporal feature vector in $\bm{S}$ 
	  and the ground-truth timing vector $\bm{d}^{\mathrm{REST}}$ under various $\alpha, \kappa$ are shown.}
	  \label{fig:correlation2}
	 \end{center}
\end{figure}

To verify the significance of sparsity, the kurtosis of the estimated temporal feature matrix $\bm S$ and the reconstruction error by our method are also summarized in Table \ref{tab:SKR}. For comparison, 
${\rm MAK}(\bm S)$ and ${\rm RMSE}_{\bm X}$ obtained by FastICA are 12.37 and 0.904, respectively. 
Recall that the objective of this experiment is to identify the temporal features response to external stimuli.
Hence, this result suggests that the sparsity of spatial features is more important than the kurtosis of temporal features or the reconstruction error in feature extraction of neuronal activity data. 
This fact can be reinforced by the result of previous study \cite{ET2024}, where they applied sparse MF methods, 
namely SparsePCA and method of optimal directions (MOD), to the same dataset and investigated the performance of feature extraction and reconstruction error. As a result, even under the case of the large reconstruction error, 
appropriate features are obtained if the sparsity of the spatial feature in the brain is large. 
Our result of experiment shows that the accuracy of synchronization between extracted temporal features and the timing of showing images can be improved by sparsifying spatial features in the brain, even if the kurtosis of temporal features is decreased. Therefore, it can be concluded that the sparsity of spatial features in the brain 
is the most important to obtain features corresponding to visual stimuli.

For comparison, we also apply another sparse ICA method in the previous work \cite{harada2021sparse} 
to the same data. 
In their method, the sparsity of the matrix ${\bm W}^T {\bm Q}$ is controlled by two parameters 
$\lambda, \alpha$, and we set $\lambda=0.1, 0.5, 1.0, \alpha=10^{-4},10^{-3},10^{-2}$, respectively. 
The result in Figure \ref{fig:correlation2} shows that very large value of Correlation over 0.5 cannot be obtained,
 in other words appropriate feature cannot be extracted sufficiently by the ICA method of sparsifying the matrix
${\bm W}^T {\bm Q}$. In contrast, our method sparsifying the matrix $\bm Q^{\#} \bm W$
yields very large value of Correlation and works more appropriately for feature extraction from fMRI data.

It should be mentioned that there is also an analysis for the same data by MOD \cite{ET2024}. 
In the previous study, it is claimed that MOD with high sparsity setting can extract the activations in the cerebellum. Comparing the results by the proposed method and MOD, it is observed that 
the spatial maps from the proposed method and MOD are similar. However, the maximum correlation between the temporal vector and visual stimuli (or resting states) by MOD is 0.716 under $K=20$. Therefore, it can be concluded that our method has advantage over MOD under the same value of $K$.
In addition, the previous study also showed that spatial ICA and SparsePCA can extract significant features. 
In the spatial features extracted by these methods,
strong activations are observed in the region of early visual cortex, which differ from the one obtained 
by MOD or our proposed method. 
Such differences should be interpreted as the effectiveness of all methods (MOD, SparsePCA, spatial ICA, and our method) in extracting features from neural activity data, and the comparison among these methods does not 
make sense so much.
In conclusion, this fact also supports the validity of our method for feature extraction from fMRI data. 

\section{Conclusion}
\label{sec:Conclusion}

In this study, we propose a novel ICA method giving sparse factorized matrix by adding an $\ell_1$ regularization term.
We also evaluate its performance by the application both to synthetic and real-world data. 
From the result by numerical experiment, 
we expect that the proposed method gains the interpretability of the result in comparison with the conventional ICA, 
because our method can give sparse factorized matrix by appropriate tuning of parameters. 
Furthermore, in the application to task-related fMRI data, our method can discriminate resting and non-resting states,
and it is competitive with MOD or other MF methods.
This indicates the utility of our proposed method in practical analysis of biological data.

As future works, we will compare the performance of our proposed method with other ICA methods,
in particular stochastic ICA \cite{donnat2019constrained} among them.
Due to its stochastic nature, it may be difficult to obtain a sparse factorized matrix stably.
However, the stochastic method may have advantages over our method, for example on scalability 
by appropriate parameter tuning or data preprocessing.
It is also necessary to apply our method to other real-world data such as genetic or financial ones, 
and investigate its utility.
In particular, in the analysis of gene expression, it is reported that ICA with the assumption of 
independence on temporal features is suitable for gene clustering
\cite{nascimento2017independent, kim2008independent}. 
In addition, it is known that the interpretation of gene clusters becomes easy 
by imposing sparsity on gene features \cite{SparsePCA2}. 
Hence, our method will be helpful for finding a novel feature of gene expression.
Finally, we hope that our method is found to be useful in various fields and can contribute to feature extraction
in many practical problems.

\section*{Acknowledgments}
The authors are grateful to Kazuharu Harada for sharing his related work \cite{harada2021sparse} 
and offering program code of sparse ICA in his work.
This work is supported by KAKENHI Nos. 18K11175, 19K12178, 20H05774, 20H05776, and 23K10978.

\bibliographystyle{unsrt}
\bibliography{article}

\begin{thebibliography}{10}

\bibitem{ICA}
Pierre Comon.
\newblock Independent component analysis, a new concept?
\newblock {\em Signal processing}, 36(3):287--314, 1994.

\bibitem{NMF}
Daniel Lee and H~Sebastian Seung.
\newblock Algorithms for non-negative matrix factorization.
\newblock {\em Advances in neural information processing systems}, 13, 2000.

\bibitem{MOD}
K.~Engan, S.O. Aase, and J.~Hakon~Husoy.
\newblock Method of optimal directions for frame design.
\newblock {\em IEEE International Conference on Acoustics, Speech, and Signal
  Processing. Proceedings. ICASSP99}, 5:2443--2446, 1999.

\bibitem{KSVD}
Michal Aharon, Michael Elad, and Alfred Bruckstein.
\newblock K-svd: An algorithm for designing overcomplete dictionaries for
  sparse representation.
\newblock {\em IEEE Transactions on signal processing}, 54(11):4311--4322,
  2006.

\bibitem{SparsePCA1}
Ian~T. Jolliffe, Nickolay~T. Trendafilov, and Mudassir Uddin.
\newblock A modified principal component technique based on the lasso.
\newblock {\em Journal of Computational and Graphical Statistics},
  12(3):531--547, 2003.

\bibitem{SparsePCA2}
Hui Zou, Trevor Hastie, and Robert Tibshirani.
\newblock Sparse principal component analysis.
\newblock {\em Journal of computational and graphical statistics},
  15(2):265--286, 2006.

\bibitem{SparseNMF}
Patrik~O. Hoyer.
\newblock Non-negative matrix factorization with sparseness constraints.
\newblock {\em J. Mach. Learn. Res.}, 5:1457--1469, 2004.

\bibitem{MF2fMRIexA}
Surya Ganguli and Haim Sompolinsky.
\newblock Compressed sensing, sparsity, and dimensionality in neuronal
  information processing and data analysis.
\newblock {\em Annual review of neuroscience}, 35:485--508, 2012.

\bibitem{MF2fMRIexB}
Xiaoyu Ding, Jong-Hwan Lee, and Seong-Whan Lee.
\newblock Performance evaluation of nonnegative matrix factorization algorithms
  to estimate task-related neuronal activities from fmri data.
\newblock {\em Magnetic resonance imaging}, 31(3):466--476, 2013.

\bibitem{MF2fMRIexC}
Jinglei Lv, Binbin Lin, Qingyang Li, Wei Zhang, Yu~Zhao, Xi~Jiang, Lei Guo,
  Junwei Han, Xintao Hu, Christine Guo, et~al.
\newblock Task fmri data analysis based on supervised stochastic coordinate
  coding.
\newblock {\em Medical image analysis}, 38:1--16, 2017.

\bibitem{MF2fMRIexD}
Michael Beyeler, Emily~L Rounds, Kristofor~D Carlson, Nikil Dutt, and Jeffrey~L
  Krichmar.
\newblock Neural correlates of sparse coding and dimensionality reduction.
\newblock {\em PLoS computational biology}, 15(6):e1006908, 2019.

\bibitem{ICA2rsfMRI}
Christian~F Beckmann, Marilena DeLuca, Joseph~T Devlin, and Stephen~M Smith.
\newblock Investigations into resting-state connectivity using independent
  component analysis.
\newblock {\em Philosophical Transactions of the Royal Society B: Biological
  Sciences}, 360(1457):1001--1013, 2005.

\bibitem{ET2024}
Yusuke Endo and Koujin Takeda.
\newblock Performance evaluation of matrix factorization for fmri data.
\newblock {\em Neural Computation}, 36(1):128--150, 2024.

\bibitem{TAO1986249}
Pham~Dinh Tao and El~Bernoussi Souad.
\newblock Algorithms for solving a class of nonconvex optimization problems.
  methods of subgradients.
\newblock In J.-B. Hiriart-Urruty, editor, {\em Fermat Days 85: Mathematics for
  Optimization}, volume 129 of {\em North-Holland Mathematics Studies}, pages
  249--271. North-Holland, 1986.

\bibitem{tao1988duality}
Pham~Dinh Tao and El~Bernoussi Souad.
\newblock Duality in dc (difference of convex functions) optimization.
  subgradient methods.
\newblock In {\em Trends in Mathematical Optimization: 4th French-German
  Conference on Optimization}, pages 277--293. Springer, 1988.

\bibitem{DC1}
Le~Thi~Hoai An and Pham~Dinh Tao.
\newblock The dc (difference of convex functions) programming and dca revisited
  with dc models of real world nonconvex optimization problems.
\newblock {\em Annals of operations research}, 133:23--46, 2005.

\bibitem{FastICA}
Aapo Hyv{\"a}rinen and Erkki Oja.
\newblock Independent component analysis: algorithms and applications.
\newblock {\em Neural networks}, 13(4-5):411--430, 2000.

\bibitem{zibulevsky2001blind1}
Michael Zibulevsky, Pavel Kisilev, Yehoshua Zeevi, and Barak Pearlmutter.
\newblock Blind source separation via multinode sparse representation.
\newblock {\em Advances in neural information processing systems}, 14, 2001.

\bibitem{zibulevsky2001blind2}
Michael Zibulevsky and Barak~A. Pearlmutter.
\newblock Blind source separation by sparse decomposition in a signal
  dictionary.
\newblock {\em Neural Computation}, 13(4):863--882, 2001.

\bibitem{bronstein2005sparse}
Alexander~M Bronstein, Michael~M Bronstein, Michael Zibulevsky, and Yehoshua~Y
  Zeevi.
\newblock Sparse ica for blind separation of transmitted and reflected images.
\newblock {\em International Journal of Imaging Systems and Technology},
  15(1):84--91, 2005.

\bibitem{harada2021sparse}
Kazuharu Harada and Hironori Fujisawa.
\newblock Sparse estimation of linear non-gaussian acyclic model for causal
  discovery.
\newblock {\em Neurocomputing}, 459:223--233, 2021.

\bibitem{chen2019sparse}
Ying Chen, Linlin Niu, Ray-Bing Chen, and Qiang He.
\newblock Sparse-group independent component analysis with application to yield
  curves prediction.
\newblock {\em Computational Statistics \& Data Analysis}, 133:76--89, 2019.

\bibitem{zhang2006ica}
Kun Zhang and Lai-Wan Chan.
\newblock Ica with sparse connections.
\newblock {\em Intelligent Data Engineering and Automated Learning}, pages
  530--537, 2006.

\bibitem{donnat2019constrained}
Claire Donnat, Leonardo Tozzi, and Susan Holmes.
\newblock Constrained bayesian ica for brain connectome inference.
\newblock {\em arXiv preprint arXiv:1911.05770}, 2019.

\bibitem{FOUCART2023441}
Simon Foucart.
\newblock The sparsity of lasso-type minimizers.
\newblock {\em Applied and Computational Harmonic Analysis}, 62:441--452, 2023.

\bibitem{OMP1}
Y.C. Pati, R.~Rezaiifar, and P.S. Krishnaprasad.
\newblock Orthogonal matching pursuit: recursive function approximation with
  applications to wavelet decomposition.
\newblock {\em Proceedings of 27th Asilomar Conference on Signals, Systems and
  Computers}, 1:40--44, 1993.

\bibitem{OMP2}
Geoffrey~M. Davis, Stephane~G. Mallat, and Zhifeng Zhang.
\newblock Adaptive time-frequency decompositions.
\newblock {\em Optical Engineering}, 33:2183--2191, 1994.

\bibitem{DC2}
Katsuya Tono, Akiko Takeda, and Jun-ya Gotoh.
\newblock Efficient dc algorithm for constrained sparse optimization.
\newblock {\em arXiv preprint arXiv:1701.08498}, 2017.

\bibitem{gLasso}
Ryan~J. Tibshirani and Jonathan Taylor.
\newblock {The solution path of the generalized lasso}.
\newblock {\em The Annals of Statistics}, 39(3):1335 -- 1371, 2011.

\bibitem{ADMM}
Stephen Boyd, Neal Parikh, Eric Chu, Borja Peleato, Jonathan Eckstein, et~al.
\newblock Distributed optimization and statistical learning via the alternating
  direction method of multipliers.
\newblock {\em Foundations and Trends in Machine learning}, 3(1):1--122, 2011.

\bibitem{tao1997convex}
Pham~Dinh Tao and LT~Hoai An.
\newblock Convex analysis approach to dc programming: theory, algorithms and
  applications.
\newblock {\em Acta mathematica vietnamica}, 22(1):289--355, 1997.

\bibitem{AmariDistance}
Shun-ichi Amari, Andrzej Cichocki, and Howard Yang.
\newblock A new learning algorithm for blind signal separation.
\newblock {\em Advances in neural information processing systems}, 8, 1995.

\bibitem{KT2023}
Ryota Kawasumi and Koujin Takeda.
\newblock Automatic hyperparameter tuning in sparse matrix factorization.
\newblock {\em Neural Computation}, 35(6):1086--1099, 2023.

\bibitem{Haxbydata}
J~V Haxby, M~I Gobbini, M~L Furey, A~Ishai, J~L Schouten, and P~Pietrini.
\newblock Distributed and overlapping representations of faces and objects in
  ventral temporal cortex.
\newblock {\em Science}, 293:2425--2430, 2001.

\bibitem{sMF2fMRIexA}
Wei Zhang, Jinglei Lv, Xiang Li, Dajiang Zhu, Xi~Jiang, Shu Zhang, Yu~Zhao, Lei
  Guo, Jieping Ye, Dewen Hu, et~al.
\newblock Experimental comparisons of sparse dictionary learning and
  independent component analysis for brain network inference from fmri data.
\newblock {\em IEEE transactions on biomedical engineering}, 66(1):289--299,
  2018.

\bibitem{sMF2fMRIexB}
Jianwen Xie, Pamela~K Douglas, Ying~Nian Wu, Arthur~L Brody, and Ariana~E
  Anderson.
\newblock Decoding the encoding of functional brain networks: An fmri
  classification comparison of non-negative matrix factorization (nmf),
  independent component analysis (ica), and sparse coding algorithms.
\newblock {\em Journal of neuroscience methods}, 282:81--94, 2017.

\bibitem{nascimento2017independent}
Moyses Nascimento, Fabyano Fonseca~e Silva, Thelma Safadi, Ana Carolina~Campana
  Nascimento, Talles Eduardo~Maciel Ferreira, La{\'\i}s Mayara~Azevedo Barroso,
  Camila Ferreira~Azevedo, Simone Eliza~Faccione Guimar{\~a}es, and Nick
  Vergara~Lopes Ser{\~a}o.
\newblock Independent component analysis (ica) based-clustering of temporal
  rna-seq data.
\newblock {\em PloS one}, 12(7):e0181195, 2017.

\bibitem{kim2008independent}
Sookjeong Kim, Jong~Kyoung Kim, and Seungjin Choi.
\newblock Independent arrays or independent time courses for gene expression
  time series data analysis.
\newblock {\em Neurocomputing}, 71(10-12):2377--2387, 2008.

\end{thebibliography}

\end{document}